\documentclass{article}

\usepackage[final, nonatbib]{neurips_2020}

\usepackage[utf8]{inputenc} 
\usepackage[T1]{fontenc}    
\usepackage{xcolor}
\usepackage{url}            
\usepackage{booktabs}       
\usepackage{amsfonts,amsmath}       
\usepackage{nicefrac}       
\usepackage{microtype}      
\usepackage{graphicx}
\usepackage{wrapfig}
\usepackage{amsthm}
\usepackage{amsmath}
\usepackage{amsfonts}
\usepackage{mathtools}
\usepackage{graphicx}
\usepackage{pdfpages}

\usepackage{xcolor}
\usepackage{tikz}
\usepackage{wrapfig}
\usepackage{subfigure}
\usepackage{pgffor}
\usepackage{float}
\usepackage{textcomp}
\usepackage{caption}

\usepackage{amssymb}
\usepackage{multicol}
\usepackage{url}

\usepackage{caption}        

\newtheorem{theorem}{Theorem}[section]

\newtheorem{lemma}[theorem]{Lemma}
\newtheorem*{lemma*}{Lemma}
\newtheorem*{theorem*}{Theorem}
\newcommand{\vol}{\mathrm{vol}}

\newcommand{\Pro}[1]{\mathbb{P} \left[\,#1\,\right]}

\newcommand{\unit}{\omega_{\lceil 2\pi\cdot k\rceil}}

\newcommand{\Ex}[1]{\mathbb{E} \left[\,#1\,\right]}

\newcommand{\calL}{\mathcal{L}}
\newcommand{\Span}{\operatorname{span}}
\newcommand{\Dim}{\operatorname{dim}}

\newcommand{\APT}{\mathsf{APT}}
\newcommand{\eps}{\epsilon}

\renewcommand{\leq}{\leqslant}
\renewcommand{\geq}{\geqslant}
\renewcommand{\le}{\leqslant}
\renewcommand{\ge}{\geqslant}

\newcommand{\lemref}[1]{Lemma~\ref{lem:#1}}

\newcommand{\eq}[1]{\eqref{eq:#1}}

\renewcommand{\tilde}{\widetilde}
\renewcommand{\epsilon}{\varepsilon}

\usepackage{mathtools}

\DeclarePairedDelimiter\abs{\lvert}{\rvert}%
\DeclarePairedDelimiter\norm{\lVert}{\rVert}%

\makeatletter
\let\oldabs\abs
\def\abs{\@ifstar{\oldabs}{\oldabs*}}
\let\oldnorm\norm
\def\norm{\@ifstar{\oldnorm}{\oldnorm*}}
\makeatother

\title{Higher-Order Spectral Clustering of Directed Graphs}

\author{%
  Steinar Laenen \thanks{steinar9@gmail.com} \\
  Oxford Research Group \\
  FiveAI \\
  \texttt{steinar.laenen@five.ai} \\
  \And    
  He Sun \\
  School of Informatics\\
  University of Edinburgh\\
  \texttt{h.sun@ed.ac.uk} \\

}

\begin{document}

\maketitle

\newcommand{\todo}[1]{\PackageWarning{}{Comment: #1}{\color{red} [{TODO:} #1]}}
\newcommand{\drafty}[1]{\PackageWarning{}{Comment: #1}{\color{gray}{#1}}}
\newcommand{\smallpar}[1]{\smallskip\noindent {\bf{#1}}}

\newcommand{\eg}{e.g.\xspace}
\newcommand{\ie}{i.e.\xspace}

\newcommand{\UA}{$\uparrow$\@}
\newcommand{\LA}{$\leftarrow$\@}
\newcommand{\DA}{$\downarrow$\@}
\newcommand{\RA}{$\rightarrow$\@}

\newenvironment{tight_enum}{
\begin{enumerate}
        \setlength{\itemsep}{0pt}
        \setlength{\parskip}{0pt}
        \setlength{\parsep}{0pt}
}{\end{enumerate}}

\newenvironment{tight_it}{
\begin{itemize}
        \setlength{\itemsep}{0pt}
        \setlength{\parskip}{0pt}
        \setlength{\parsep}{0pt}
}{\end{itemize}}

\begin{abstract}
Clustering is an important topic in algorithms, and has a number of applications in machine learning, computer vision, statistics, and several other research disciplines.  Traditional objectives of graph clustering are to find clusters with low conductance. 
Not only are these objectives just applicable for undirected graphs, they are also incapable to take the relationships between clusters into account, which could be crucial for many applications.
To overcome these downsides, we study directed graphs (digraphs) whose clusters exhibit further ``structural'' information amongst each other. Based on the Hermitian matrix representation of digraphs, we present a nearly-linear time algorithm for digraph clustering, and further show that our proposed  algorithm can be implemented in sublinear time under reasonable assumptions. 
   The significance of our theoretical work is demonstrated by extensive experimental results on the UN Comtrade Dataset: the output clustering of our algorithm exhibits not only  how the clusters~(sets of countries) relate to each other with respect to their import and  export records, but also how these clusters evolve over time, in accordance with known facts in international trade.

\end{abstract}

\section{Introduction}

 Clustering is one of the most fundamental problems in algorithms and has applications in many research fields including machine learning, 
 network analysis, and  statistics.  Data can often be  represented by a graph~(e.g., users in a social network, servers in a communication network), and this makes graph clustering a natural choice to analyse these datasets. Over the past three decades, most studies on undirected graph clustering have focused on the task of partitioning with respect to the edge densities, i.e., vertices form a cluster if they are better connected to each other than to the rest of the graph. 
   The well-known normalised cut value~\cite{shi2000normalized} and graph conductance~\cite{LeeGT14} capture these  
   classical definitions of clusters,  and have become the objective functions of most undirected graph clustering algorithms.

 While the design of these algorithms has received a lot of research attention from both theoretical and applied research areas, these algorithms are usually unable to uncover \emph{higher-order structural} information among clusters in \emph{directed graphs}~(digraphs). 
 For example, let us look at the international oil trade network~\cite{screenshot_uncomtrade}, which employs digraphs to represent how mineral fuels and  oils are imported and exported between countries. Although this highly connected digraph presents little cluster structure with respect to a typical objective function of undirected graph clustering, from an economic point of view  this digraph clearly exhibits a structure of clusters:   there is a cluster of countries mainly exporting oil, a cluster mainly importing oil, and several clusters in the middle of this trade chain. All these clusters are characterised by the imbalance of the edge directions between clusters, and further present a clear ordering reflecting the overall trade pattern. This type of structure is not only found in trade data, but also in many other types of data such as migration data and infectious disease spreading data. We view these types of patterns as a \emph{higher-order} structure among the clusters and, in our point of view, this
 structural information could  be as important as the individual clusters themselves.

\paragraph{Our contribution.} In this work we study  clustering algorithms for digraphs whose cluster structure is defined with respect to the imbalance of  edge densities as well as the edge directions between the clusters.  Formally, for any  set of vertices $S_0,\ldots, S_{k-1}$ that forms a partition of the vertex set $V(G)$ of a digraph $G$, we define the \emph{flow ratio} of $\{S_j\}_{j=0}^{k-1}$ by
\[
\sum_{j=1}^{k-1} \frac{w(S_j, S_{j-1})}{\vol(S_j) + \vol(S_{j-1})},
\]
where $w(S,T) \triangleq \sum_{ \substack{ (u,v)\in E \\ u\in S, v\in T} } w(u,v) $ is the cut value from $S\subset V$ to $T\subset V$  and   $\vol(S)$ is the sum of degrees of  the vertices in $S$. We  say that 
$\{S_j\}_{j=0}^{k-1}$ forms an optimal partition if this $\{S_j\}_{j=0}^{k-1}$ \emph{maximises} the flow ratio over all possible partitions. 
By introducing a complex-valued   representation of the graph Laplacian matrix $\mathcal{L}_G$, we show that this optimal partition $\{S_j\}_{j=0}^{k-1}$  is well embedded into the bottom eigenspace of $\mathcal{L}_G$.
To further exploit this novel and intriguing connection, we show that an approximate partition with bounded approximation guarantee can be computed in time nearly-linear in the number of edges of the input graph. In the settings for which the degrees of the vertices are known in advance, we also present a sub-linear time implementation of the algorithm.  The significance of our work is further demonstrated by experimental results on   several synthetic and real-world datasets. In particular,  on the UN Comtrade dataset our clustering results are well supported by the literature from other research fields.   
At the technical level, our analysis could be viewed as a hybrid between the proof of the Cheeger inequality~\cite{Chung97}   and the analysis of spectral clustering for undirected graphs~\cite{PSZ17}, as well as a sequence of recent work on fast constructions of graph sparsification~(e.g., \cite{siamcomp/SpielmanS11}). We believe our analysis for the new Hermitian Laplacian $\mathcal{L}_G$ could inspire future research on studying the clusters' higher-order structure using spectral methods. 

\paragraph{Related work.} There is a rich literature on spectral algorithms for graph clustering. For undirected graph clustering, the works most related to ours are \cite{PSZ17,shi2000normalized,von2007tutorial}. For  digraph clustering, \cite{Satuluri:2011:SCD:1951365.1951407} proposes to perform spectral clustering on the symmetrised matrix  
$A=M^{\intercal}M + MM^{\intercal}$ of the input graph's adjacency matrix $M$;  \cite{HeDirected} initiates the studies of spectral clustering on complex-valued Hermitian matrix representations of digraphs, however their theoretical analysis only holds for digraphs generated from the stochastic block model.  Our work is also linked to analysing higher-order structures of clusters  in undirected graphs~\cite{BensonGL15,Benson163,YinBLG17}, and community detection in digraphs~\cite{Chung05,PRL}. 
The main takeaway   is that there is no previous work which analyses digraph spectral clustering algorithms to uncover the higher-order structure of clusters in a general digraph.

\label{sec:introduction}


\section{Preliminaries}
Throughout the paper, we always assume that 
$G=(V,E,w)$ is a digraph with $n$ vertices, $m$ edges, and   weight function $w : V \times V \rightarrow \mathbb{R}_{\geq 0}$. We write $u\leadsto v$ if there is an edge from $u$ to $v$ in the graph. 
For any vertex $u$, the in-degree and out-degree of $u$ are defined as $d_u^{\text{in}}\triangleq \sum_{v: v\leadsto u}w(v,u)$ and $d_u^{\text{out}} \triangleq \sum_{v: u\leadsto v} w(u,v)$, respectively. We further define the total degree of $u$ by $d_u \triangleq d_u^{\text{in}} + d_u^{\text{out}}$, and define $\vol(S)\triangleq \sum_{u\in S} d_u$ for any $S\subseteq V$.
For any set of vertices $S$ and $T$, the symmetric difference between $S$ and $T$ is defined by  $S \bigtriangleup T \triangleq (S \setminus T) \cup (T \setminus S)$.

Given any digraph  $G$ as input, we use  $M\in\mathbb{R}^{n \times n}$  to denote the adjacency matrix of $G$, where   $M_{u,v} = w(u,v)$ if there is an edge $u\leadsto v$, and 
$M_{u,v}=0$ otherwise. We use $A\in \mathbb{C}^{n\times n}$ to represent the Hermitian adjacency matrix of $G$, where $A_{u,v} = \overline{A_{v,u}} =w(u,v)\cdot\omega_{\lceil 2\pi k \rceil}  $ if $u\leadsto v$, and $A_{u,v} = 0$
otherwise. Here, $\omega_{\lceil 2\pi k \rceil} $ is the $\lceil 2\pi k \rceil$-th root of unity, and $\overline{x}$ is the conjugate of $x$. The normalised Laplacian matrix of $G$ is defined by $\mathcal{L}_G \triangleq I - D^{-1/2}AD^{-1/2}$, where the degree matrix $D\in\mathbb{R}^{n\times n}$ is defined by $D_{u,u}=d_u$, and $D_{u,v}=0$ for any  $u\neq v$.
 We sometimes drop the subscript $G$ if the underlying graph is clear from the context.

For any Hermitian matrix $A \in \mathbb{C}^{n \times n}$ and non-zero vector $x \in \mathbb{C}^{n}$, the Rayleigh quotient $\mathcal{R}(A, x)$ is defined as $\mathcal{R}(A, x) \triangleq x^{*}Ax / x^{*}x$, where $x^{*}$ is the complex conjugate transpose of  $x \in \mathbb{C}^{n}$. 
 For any Hermitian matrix $B \in \mathbb{C}^{n \times n}$, let $\lambda_1(B) \leq \ldots \leq \lambda_n(B)$ be the eigenvalues of $B$ with corresponding eigenvectors $f_1, \ldots, f_n$, where $f_j\in\mathbb{C}^n$ for any $1\leq j\leq n$.  
      
\label{sec:preliminaries}

\section{Encoding the flow-structure into $\mathcal{L}_G$'s bottom eigenspace}\label{sec:encode_flow}
 
Now we study the structure of clusters with respect to their flow imbalance,  and their relation to the bottom eigenspace of the normalised   Hermitian Laplacian matrix. For any set of vertices $S_0,\ldots, S_{k-1}$, we say that $S_0,\ldots, S_{k-1}$ form a $k$-way partition of $V(G)$, if it holds that $\bigcup_{0\leq j\leq k-1} S_j = V(G)$ and $S_j \cap S_{\ell}=\emptyset$ for any $j\neq \ell$. As discussed in Section~\ref{sec:introduction}, the primary focus of the paper is to study digraphs in which there are significant connections from $S_j$ to $S_{j-1}$ for any $1\leq j\leq k-1$. To formalise this, we introduce the notion of \emph{flow ratio} of    $\{S_j\}_{j=0}^{k-1}$, which  is defined by
 \begin{equation}\label{eq:defphiG}
 \Phi_G\left(S_0,\ldots, S_{k-1}\right) \triangleq \sum_{j=1}^{k-1}\frac{w(S_j, S_{j-1})}{\vol(S_j) + \vol(S_{j-1})}.
 \end{equation}
 We call this $k$-way partition $\{S_j\}$ an \emph{optimal clustering} if    the flow ratio given by $\{S_j\}$ achieves the maximum defined by 
 \begin{equation}\label{eq:thetak}
\theta_k(G)\triangleq \max_{\substack{S_0,\ldots, S_{k-1}\\ \cup S_i = V, S_j\cap S_{\ell}=\emptyset}}\Phi_G\left(S_0,\ldots, S_{k-1}\right).
\end{equation}
Notice that,   for any two consecutive clusters $S_j$ and $S_{j-1}$, the value $w(S_j,S_{j-1})\cdot\left( \vol(S_j) + \vol(S_{j-1}) \right)^{-1}$ evaluates   the
ratio of the total  edge weight in the cut $(S_j, S_{j-1})$ to the total weight of the edges with endpoints in $S_j$ or $S_{j-1}$;  moreover, only $k-1$ out of  $2\cdot { k\choose 2}$ different cuts among $S_0,\ldots, S_{k-1}$ contribute to
$\Phi_G(S_0,\ldots, S_{k-1})$
according to  \eq{defphiG}. 
We remark that, although the  definition of $\Phi_G(S_0,\ldots, S_{k-1})$ shares some similarity with the normalised cut value for undirected graph clustering~\cite{shi2000normalized}, in our setting an optimal clustering is the one that 
\emph{maximises} the flow ratio. This is in a sharp contrast to most objective functions for undirected graph clustering, whose aim is to find clusters of low conductance\footnote{It is important to notice that, among $2\cdot {k \choose 2}$  cuts formed by pairwise different clusters, only $(k-1)$ cut values contribute to our objective function. If one takes all of the $2\cdot {k \choose 2}$ cut values into account, the objective function would involve $2\cdot {k \choose 2}$ terms. However, even if most of the $2\cdot {k \choose 2}$ terms are much smaller than the ones along the flow, their sum could still be dominant, leaving little information on the structure of clusters.
Therefore, we should only take $(k-1)$ cut values   into account when the clusters present a flow structure. }. In addition, it is not difficult to show that this problem is $\mathsf{NP}$-hard since, when $k=2$, our problem is exactly the MAX DICUT problem studied in 
\cite{GoemansW95}.
 
To study the relationship between the flow structure among   $S_0,\ldots, S_{k-1}$ and   the eigen-structure of the normalised Laplacian matrix of the graph, we define for every optimal cluster $S_j~(0\leq j\leq k-1)$ an indicator vector $\chi_j\in\mathbb{C}^n$ by $\chi_j(u)\triangleq\left( w_{\lceil 2\pi\cdot k\rceil }\right)^j$ if $u\in S_j$ and $\chi_j(u)=0$ otherwise. We further define the normalised indicator vector of $\chi_j$ by 
\[
\widehat{\chi_j} \triangleq\frac{D^{1/2} \chi_j}{\|D^{1/2} \chi_j \|},
\]
and set 
\begin{equation}\label{eq:defy}
y \triangleq \frac{1}{\sqrt{k}}\sum_{j=0}^{k-1} \widehat{\chi_j}.
\end{equation}
We highlight that, due to the use of complex numbers,   a single vector $y$ is sufficient  to encode the structure of $k$ clusters: this is quite different from the case of undirected graphs, where $k$ mutually perpendicular vectors are needed in order to study the eigen-structure of graph Laplacian and the cluster structure~\cite{LeeGT14,PSZ17,von2007tutorial}.  In addition, by the use of roots of unity in \eq{defy},  different clusters are separated from each other by   angles, indicating that the use of a single eigenvector could be sufficient to approximately recover $k$ clusters.
Our result on the relationship between $\lambda_1(\mathcal{L}_G)$ and $\theta_k(G)$ is summarised as follows:

\begin{lemma}\label{lem:boundthetak}
Let  $G = (V,E,w)$ be a weighted digraph  with normalised Hermitian Laplacian $\mathcal{L}_{G} \in \mathbb{C}^{n \times n}$. Then, it holds that $
\lambda_1(\mathcal{L}_G) \leq 1- \frac{4}{k}\cdot\theta_k(G)$.  Moreover, $\theta_k(G)=k/4$   if $G$ is a bipartite digraph with all the edges having the same direction, and 
$\theta_k(G)<k/4$ otherwise. 
\end{lemma}

 Notice that the bipartite graph $G$ with $\theta_k(G)=k/4$ is a trivial case for our problem; hence, without lose of generality  we   assume   $\theta_k(G)<k/4$ in the following analysis.  
  To   study how the distribution of eigenvalues influences the cluster structure, similar to the case of undirected graphs we introduce the parameter $\gamma$ defined by
\[
\gamma_k(G)\triangleq \frac{ \lambda_2}{ 1-(4/k)\cdot \theta_k(G)}.
\]
Our next theorem shows that the structure of clusters in $G$ and the eigenvector corresponding to $\lambda_1(\mathcal{L}_G)$ can be approximated by each other with approximation ratio inversely proportional to $\gamma_k(G)$.

\begin{theorem}
\label{thm:structure_theorem}
The following statements hold:  (1) there is some $\alpha  \in \mathbb{C}$ such that the vector $\tilde{f_1} = \alpha f_1$ satisfies  $\|y - \tilde{f_1}\|^2 \leq 1 / \gamma_k(G) $; (2) there is some $\beta  \in \mathbb{C}$ such that the vector $\tilde{y} = \beta y$ satisfies $\norm{f_1 - \tilde{y}}^2 \leq 1 / \left( \gamma_k(G) - 1\right)$. 
\end{theorem}


\section{Algorithm}\label{sec:Algorithm}

In this section we discuss the  algorithmic contribution of the paper. In Section~\ref{sec:algodes} we will  describe the main algorithm, and its efficient implementation based on nearly-linear time Laplacian solvers; we will further present a sub-linear time implementation of  our algorithm,  assuming the degrees of the vertices are known in advance.  The main technical ideas used in analysing the   algorithms will be discussed in Section~\ref{sec:algoanalsysis}. 

\subsection{Algorithm Description\label{sec:algodes}}
 
  \paragraph{Main algorithm.} We have seen from Section~\ref{sec:encode_flow}  that the structure of clusters is approximately encoded in the bottom eigenvector of $\mathcal{L}_G$. To exploit this fact, we propose to embed the vertices of $G$ into $\mathbb{R}^2$ based on the bottom eigenvector of $\mathcal{L}_G$, and apply $k$-means on the embedded points.  Our algorithm, which we call \texttt{SimpleHerm}, only consists of a few lines of code and is described as follows: (1)  compute the bottom eigenvector $f_1\in\mathbb{C}^n$ of the normalised Hermitian Laplacian matrix $\mathcal{L}_G$ of $G$; (2) compute the embedding $\{F(v)\}_{v\in V[G]}$, where
$
F(v) \triangleq \frac{1}{\sqrt{d_v}}\cdot f_1(v)
$
for any vertex $v$; (3) apply $k$-means on the embedded points $\{F(v)\}_{v\in V[G]}$. 

We remark that, although the entries of $\mathcal{L}_G$ are complex-valued, some variant of the graph Laplacian solvers could still be applied for our setting. For most practical instances,  we have $k=O(\log^c n)$ for some constant $c$,  in which regime our proposed algorithm runs in nearly-linear time\footnote{Given any graph $G$ with $n$ vertices and $m$ edges as input, we say an  algorithm runs in nearly-linear time if the algorithm's runtime is $O(m\cdot\log^c n)$ for some constant $c$.}. We refer a reader to \cite{LSZ19} on technical discussion on the algorithm of  approximating $f_1$ in nearly-linear time.

\paragraph{Speeding-up the runtime of the algorithm.}
Since $\Omega(m)$ time is needed for any algorithm to read an entire graph, the runtime of our proposed algorithm is optimal up to a poly-logarithmic factor. However we will show that,   when the vertices' degrees are available in advance, the following  sub-linear time algorithm could be applied before the execution of the main algorithm, and this will result in the algorithm's total runtime to be sub-linear in $m$.

More formally, our proposed sub-linear time implementation is to construct a sparse subgraph $H$ of the original input graph $G$, and run the main algorithm on $H$ instead. The algorithm for obtaining graph $H$ works as follows:  
every vertex $u$ in the graph $G$ checks each of its outgoing edges $e = (u,v)$,  and samples each outgoing edge with probability 
\begin{equation*}
    p_u(u,v) \triangleq \text{min}\left\{w(u,v) \cdot \frac{\alpha \cdot \log n}{\lambda_{2} \cdot d_u^{\text{out}}}, 1\right\};
\end{equation*}
in the same time, every vertex $v$ checks each of its incoming edges $e=(u,v)$  with probability
\begin{equation*}
    p_v(u,v) \triangleq \text{min}\left\{w(u,v) \cdot \frac{\alpha \cdot \log n}{\lambda_{2} \cdot d_v^{\text{in}}}, 1\right\},
\end{equation*}
where $\alpha \in \mathbb{R}_{> 0}$ is some  constant which can be determined experimentally.  
As the algorithm goes through each vertex, it maintains all the sampled edges in a set $F$. Once all the edges have been checked, the algorithm returns a   weighted graph $H = (V, F, w_H)$, where each sampled edge $e=(u,v)$ has a new weight defined by $w_H(u,v) =w(u,v)/p_e$.  
Here,  $p_e$     is the probability that $e$ is sampled by one of its endpoints and, for any $e=(u,v)$, we can write $p_e$ as      $
p_e = p_u(u,v) + p_v(u,v) - p_u(u,v)p_v(u,v)$.

\subsection{Analysis\label{sec:algoanalsysis}}

\paragraph{Analysis of the main algorithm.}

Now we analyse the proposed algorithm, and prove that running $k$-means on $\{F(v)\}_{v\in V[G]}$ is sufficient  to obtain a meaningful clustering with bounded approximation guarantee.   We assume that the output of a $k$-means algorithm is $A_0,\ldots, A_{k-1}$. We define the cost function of the output clustering $A_0,\ldots, A_{k-1}$ by
 \[
\mathsf{COST}(A_0,\ldots, A_{k-1}) \triangleq \min_{c_0,\ldots, c_{k-1}\in\mathbb{C}} \sum_{j=0}^{k-1}\sum_{u\in A_j} d_u\| F(u) - c_j\|^2,
\]
and    define the optimal clustering by
\[
\Delta_k^2 \triangleq \min_{\footnotesize{ \mbox{partition~} A_0,\ldots A_{k-1}}} \mathsf{COST}(A_0,\ldots, A_{k-1}).
\]
Although computing the optimal clustering for $k$-means is $\mathsf{NP}$-hard, we will show that the cost value for the optimal clustering can be upper bounded with respect to $\gamma_k(G)$. To achieve  this, we define $k$ points $p^{(0)},\ldots, p^{(k-1)}$ in $\mathbb{C}$, where $p^{(j)}$ is defined by
\begin{equation}
    p^{(j)} = \frac{\beta}{\sqrt{k}}\cdot\frac{(\omega_{\lceil 2\pi\cdot k\rceil})^j}{\sqrt{\vol(S_j)}}, \qquad 0\leq j\leq k-1.
\end{equation}
We could view these $p^{(0)},\ldots, p^{(k-1)}$ as approximate centers of the $k$ clusters, which are separated from each other through different powers of $\omega_{\lceil 2\pi\cdot k\rceil}$.

Our first lemma  shows that  the total distance between the embedded points from every $S_j$ and their respective   centers $p^{(j)}$ can be upper bounded, which is summarised as follows:

\begin{lemma}
\label{lem:distance_p_x}
It holds that
$    \sum_{j = 0}^{k-1}\sum_{u \in S_j} d_u\cdot \norm{F(u) - p^{(j)}}^2 \leq (\gamma_k(G) - 1)^{-1}.
$
\end{lemma}
Since the cost value of the optimal clustering is the minimum over all possible partitions of the embedded points, by Lemma~\ref{lem:distance_p_x} we have that 
$\Delta_k^2\leq(\gamma_k(G)-1)^{-1}$.
 We assume that the $k$-means algorithm used here achieves an approximation ratio of $\mathsf{APT}$.  Therefore, 
the output $A_0,\ldots, A_{k-1}$ of this $k$-means algorithm   satisfies  $\mathsf{COST}(A_0,\ldots, A_{k-1}) \leq \mathsf{APT}\big/ (\gamma_k(G)-1)$.

Secondly, we show  that the norm of the approximate centre of each cluster is \emph{inversely} proportional to the volume of each cluster. This implies that larger clusters are closer to the origin, while smaller clusters are further away from the origin. 
\begin{lemma}\label{lem:pi}
It holds for any $0\leq j\leq k-1$ that 
$
\left\| p^{(j)} \right\|^2 =  \|\beta \|^2\cdot (k \cdot  \vol(S_j))^{-1}.
$
\end{lemma}
 
Thirdly, we prove that the distance between different approximate centres $p^{(j)}$ and $p^{(\ell)}$ is inversely proportional to the volume of the \emph{smaller} cluster, which implies that the embedded points of the vertices from a smaller cluster are far from the embedded points from other clusters. This key fact explains why our algorithm is able to  approximately recover the structure of all the clusters.

\begin{lemma}\label{lem:distancepjpl}
It holds for any $0\leq j\neq \ell \leq k-1$ that
$
\left\| p^{(j)} -p^{(\ell)} \right\|^2 \geq \frac{\|\beta\|^2}{3k^3\cdot \min\{\vol(S_j), \vol(S_{\ell})\}}.
$
\end{lemma}

 
Combining these three lemmas   with some combinatorial analysis, we prove that the symmetric difference between  every returned cluster by the algorithm and its   corresponding cluster in the optimal partition can be upper bounded, since otherwise the   cost value of the returned clusters would contradict   Lemma~\ref{lem:distance_p_x}.

\begin{theorem}\label{thm:apt}
Let $G=(V,E)$ be a digraph, and $S_0,\dots, S_{k-1}$ be a $k$-way partition of $V[G]$ that maximises the flow ratio $\Phi_G(S_0,\dots, S_{k-1})$. Then,
there is an algorithm that returns  a $k$-way partition $A_0,\ldots, A_{k-1}$ of $V[G]$. Moreover, by  assuming $A_j$ corresponds to  $S_j$ in the optimal partition, it holds that $\vol(A_j\triangle S_j)\leq \varepsilon \vol(S_j)$ for some  $\varepsilon = 48 k^3\cdot (1+\mathsf{APT})\big/\left(\gamma_k(G)-1\right)\leq 1/2$. 
\end{theorem}

We remark that the analysis of our algorithm is similar with the work of \cite{PSZ17}. However, the analysis in \cite{PSZ17}  relies on $k$ indicator vectors of $k$ clusters, each of which is in a different dimension of $\mathbb{R}^n$; this implies that $k$ eigenvectors are needed in order to find a good $k$-way partition. In our case, all the embedded points are in $\mathbb{R}^2$, and the embedded points from different clusters are mainly separated by angles; this makes our analysis slightly more involved than \cite{PSZ17}.
  
\paragraph{Analysis for the speeding-up subroutine.}

We further analyse the speeding-up subroutine described in Section~\ref{sec:algodes}. Our analysis is very  similar with
\cite{SZ19}, and the approximation guarantee of our speeding-up subroutine is   as follows:

   \begin{theorem}\label{thm:distributed_sparsification}

Given a digraph $G=(V,E)$ as input, the speeding-up subroutine computes a subgraph $H=(V,F)$ of $G$ with $O((1/\lambda_2) \cdot n\log n)$ edges. Moreover, with high probability, the computed sparse graph $H$ satisfies that $\theta_k(H) = \Omega (\theta_k(G))$, and $ \lambda_2(\mathcal{L}_H) = \Omega (\lambda_2(\mathcal{L}_G))$.
  \end{theorem}

\label{sec:method}

\section{Experiments}\label{sec:experiments}
 

 In this section we present the experimental results of our proposed algorithm \texttt{SimpleHerm} on both synthetic and real-world datasets, and compare its performance against the previous state-of-the-art.  All our experiments are conducted with an ASUS ZenBook Pro UX501VW with an Intel(R) Core(TM) i7-6700HQ CPU @ 2.60GHz with 12GB of RAM.

We will compare \texttt{SimpleHerm} against the \texttt{DD-SYM} algorithm~\cite{Satuluri:2011:SCD:1951365.1951407} and the \texttt{Herm-RW} algorithm~\cite{HeDirected}. 
Given the adjacency matrix $M\in\mathbb{R}^{n\times n}$ as input, the \texttt{DD-SYM} algorithm computes the matrix $A=M^{\intercal}M + MM^{\intercal}$,  and uses the top $k$ eigenvectors of a random walk   matrix  $D^{-1}A$ to construct an embedding for $k$-means clustering. The \texttt{Herm-RW} algorithm uses the imaginary unit $i$ to represent directed edges and applies the top $\lceil k/2\rceil$ eigenvectors of a random walk   matrix to construct an embedding for $k$-means. Notice that  both of  the   \texttt{DD-SYM} and  \texttt{Herm-RW} algorithms involve the use of multiple eigenvectors, and \texttt{DD-SYM} requires computing   matrix multiplications, which makes it  computationally more expensive than ours. 

\subsection{Results on Synthetic Datasets}

We first perform experiments on graphs generated from the Directed Stochastic Block Model~(DSBM) which is introduced in \cite{HeDirected}. We introduce a path structure into the DSBM, and compare the performance of our algorithm against the others.
Specifically, for given parameters $k,n,p,q,\eta$, a graph randomly chosen from the DSBM is constructed as follows: the overall   graph consists of $k$ clusters $S_0,\ldots, S_{k-1}$ of the same size, each of which can be initially viewed as a $G(n,p)$ random graph. We connect edges with endpoints in different clusters with probability $q$, and connect edges with endpoints within the same cluster with probability $p$. In addition,
for any edge $(u,v)$ where $u\in S_{j}$ and $v\in S_{j+1}$, we set the edge direction as $u\leadsto v$ with probability $\eta$, and set the edge direction as $v\leadsto u$   with probability $1-\eta$. For all other pairs of clusters which do not lie along the path, we set their edge directions randomly. The directions of edges inside a cluster are assigned randomly. 

As graphs generated from the DSBM have a well-defined ground truth clustering, we apply the Adjusted Rand Index~(ARI)~\cite{gates2017impact} to measure the performance of different algorithms.  We further set $p=q$, since this is one of the hardest regimes for studying the DSBM. In particular, when $p=q$, the edge density plays no role in characterising the structure of clusters, and the edges are \emph{entirely} defined with respect to the edge directions.

We set $n=1000$, and $k=4$.   We set the value of $p$ to be between $0.5$ and $0.8$, and the value of $\eta$ to be between $0.5$ and $0.7$. As shown in Figure~\ref{fig:dsbm_plot}, our proposed \texttt{SimpleHerm} clearly outperforms the \texttt{Herm-RW} and the \texttt{DD-SYM} algorithms.

\begin{figure}[ht]
\centering
    \begin{minipage}{0.24\textwidth}
      \centering
    \includegraphics[width=\textwidth]{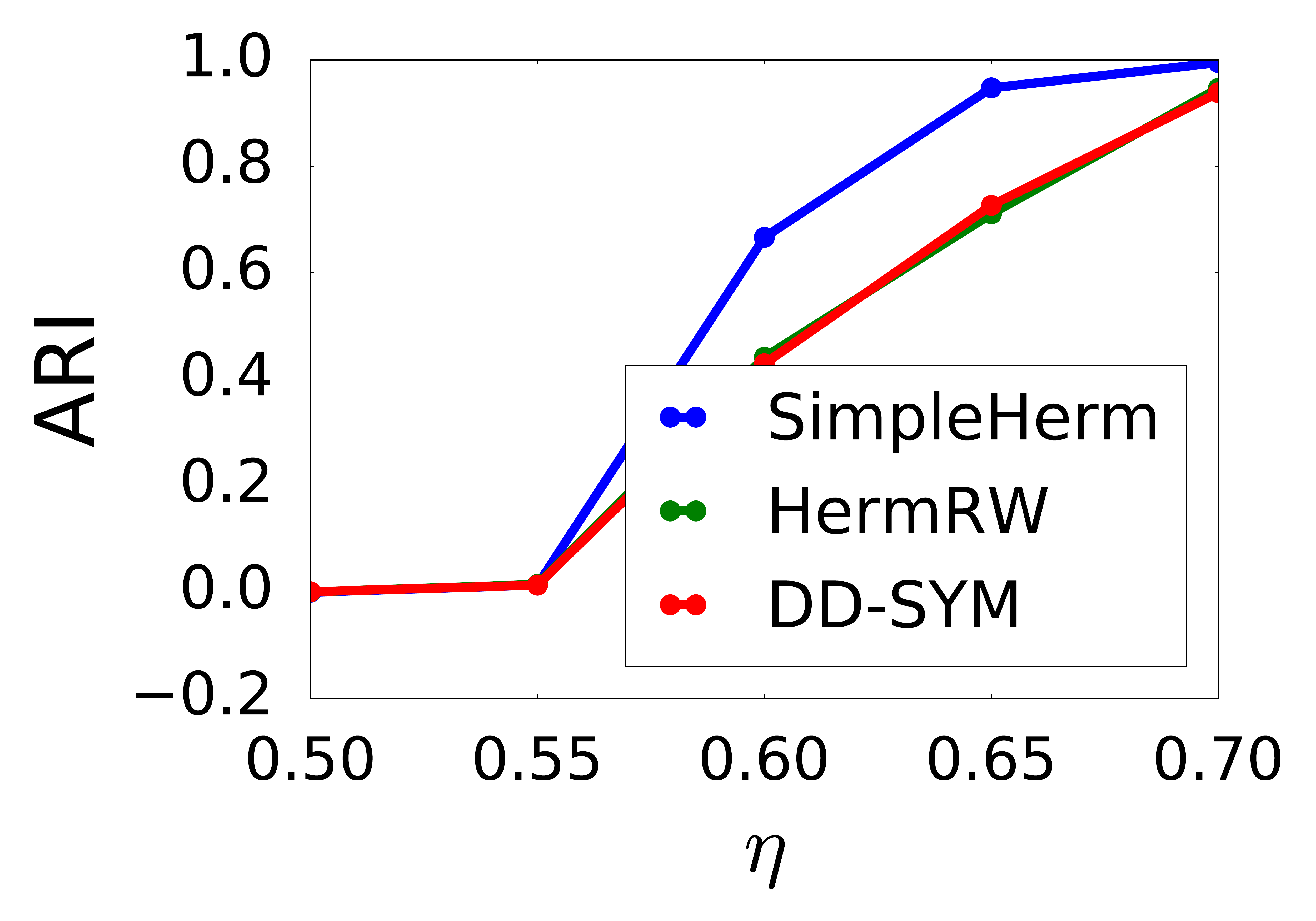}
    \vspace{-0.5cm}
    \caption*{$p=0.5$}
    \end{minipage}%
    \begin{minipage}{0.24\textwidth}
      \centering  
    \includegraphics[width=\textwidth]{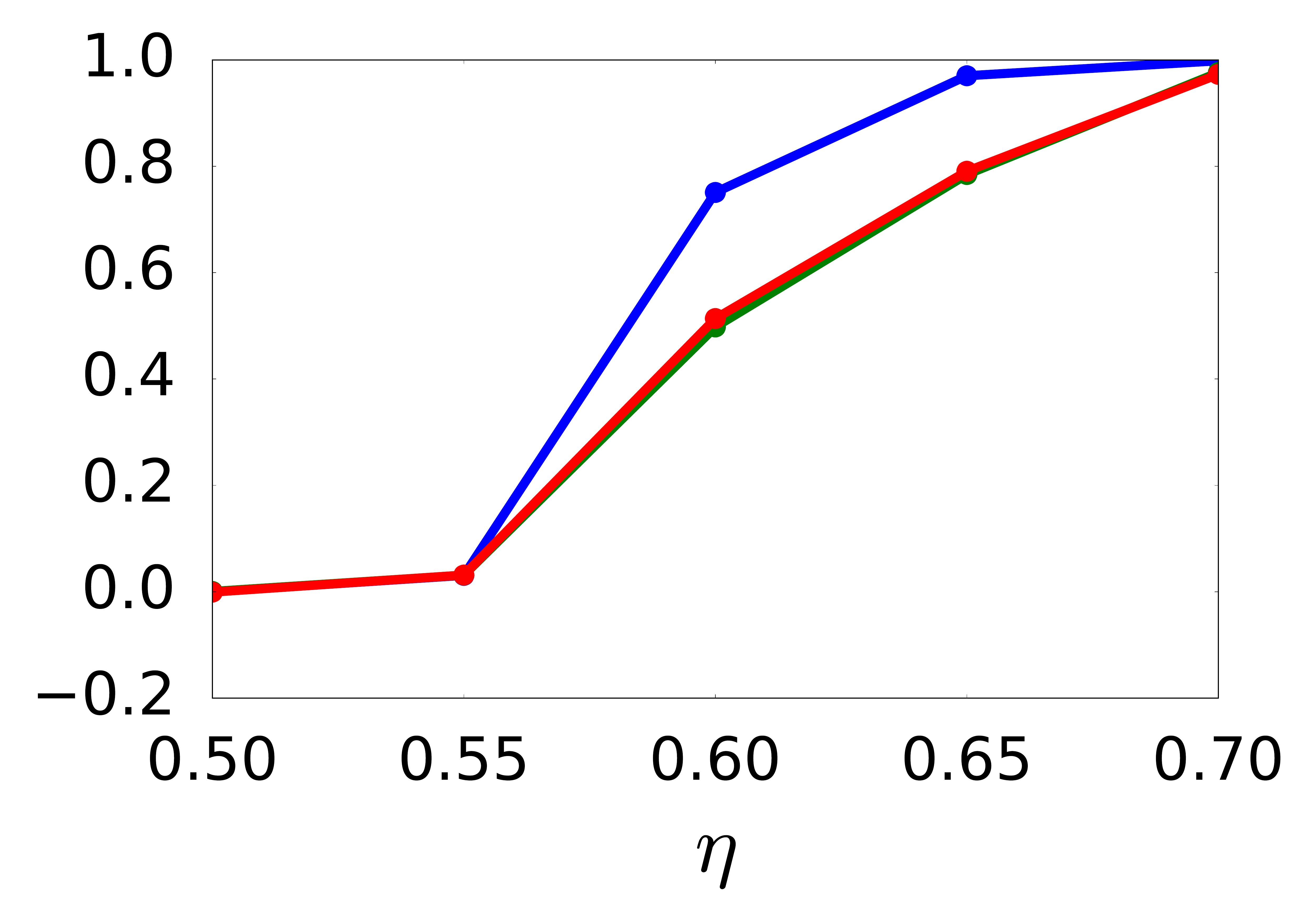}
    \vspace{-0.5cm}
    \caption*{$p=0.6$}
    \end{minipage}
    \begin{minipage}{0.24\textwidth}
      \centering
    \includegraphics[width=\textwidth]{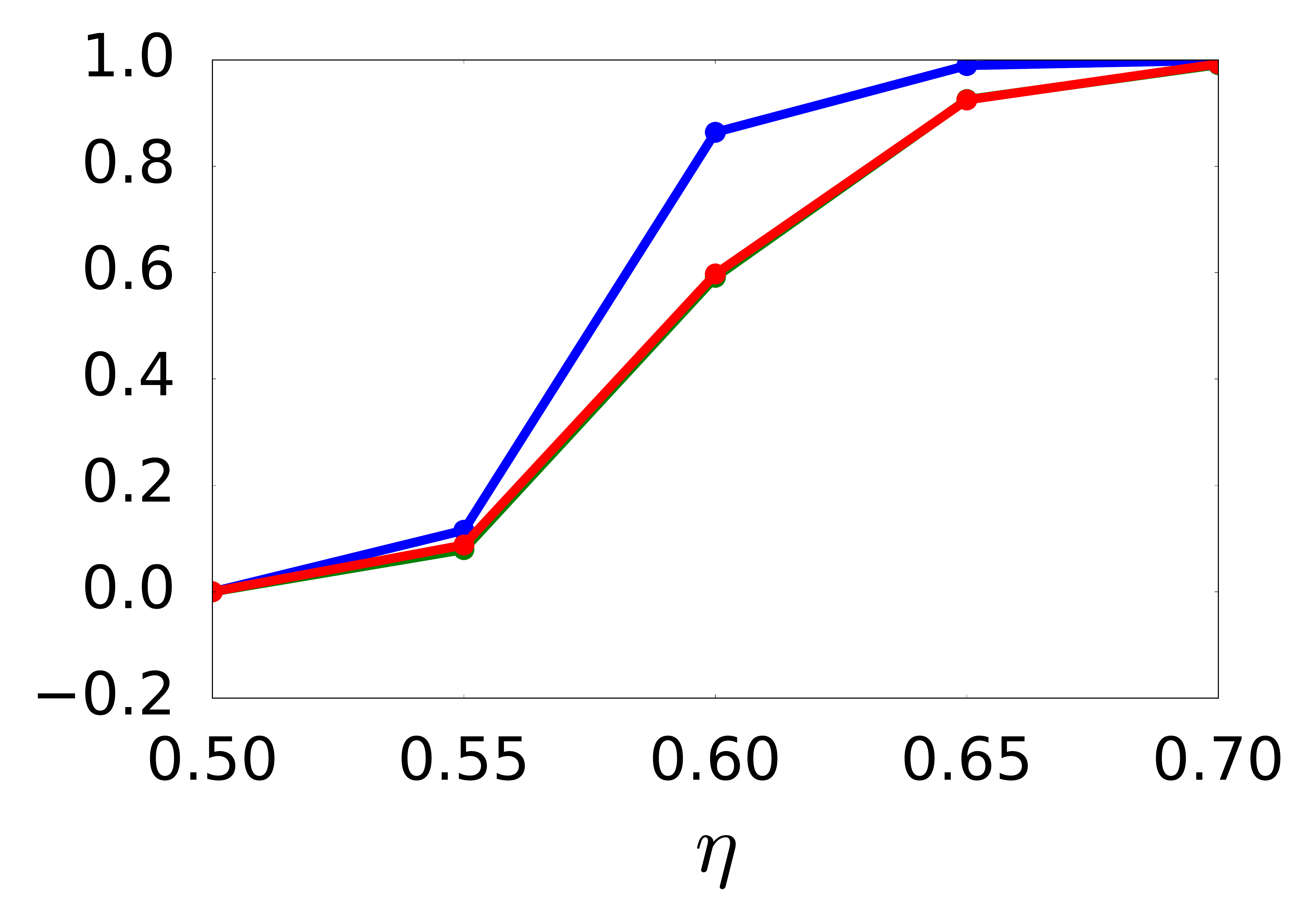}
    \vspace{-0.5cm}
    \caption*{$p=0.7$}
    \end{minipage}%
    \begin{minipage}{0.24\textwidth}
      \centering
    \includegraphics[width=\textwidth]{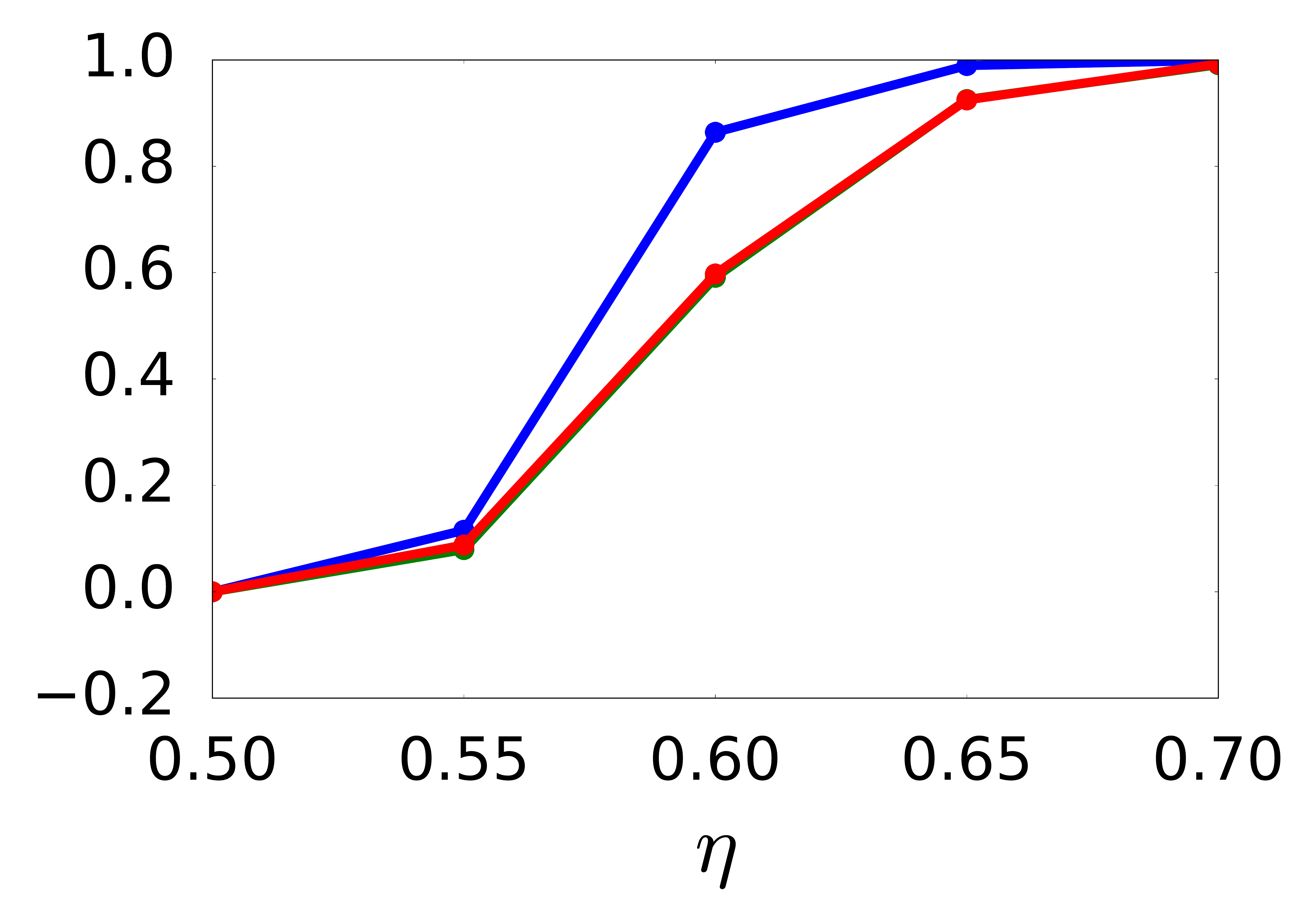}
    \vspace{-0.5cm}
    \caption*{$p=0.8$}
    \end{minipage}%
\caption{\small{$n=1000$ and $k=4$. Average ARIs   over $5$ runs of different algorithms, with respect to different values of $p$ and $\eta$.
}} \label{fig:dsbm_plot}
\end{figure}

Next, we study the case of $n=2000$ and $k=8$, but   the structure of clusters presents a more significant path topology.  Specifically, we assume that   any pair of vertices within each cluster are connected with probability $p\in(0.05,0.1)$; moreover,  all the edges crossing different clusters are along the cuts
$(S_j, S_{j+1})$  for some $0\leq j\leq k-2$. By setting $\eta\in(0.65,1)$, our results are reported in Figure~\ref{fig:dsbm_sparse_plot}.
From these results, it is easy to see that, when the underlying graph presents a clear flow structure, our algorithm performs significantly better than both the \texttt{Herm-RW} and \texttt{DD-SYM} algorithms, for which multiple eigenvectors are needed. 



\begin{figure}[h]
\centering
    \begin{minipage}{0.24\textwidth}
      \centering
    \includegraphics[width=\textwidth]{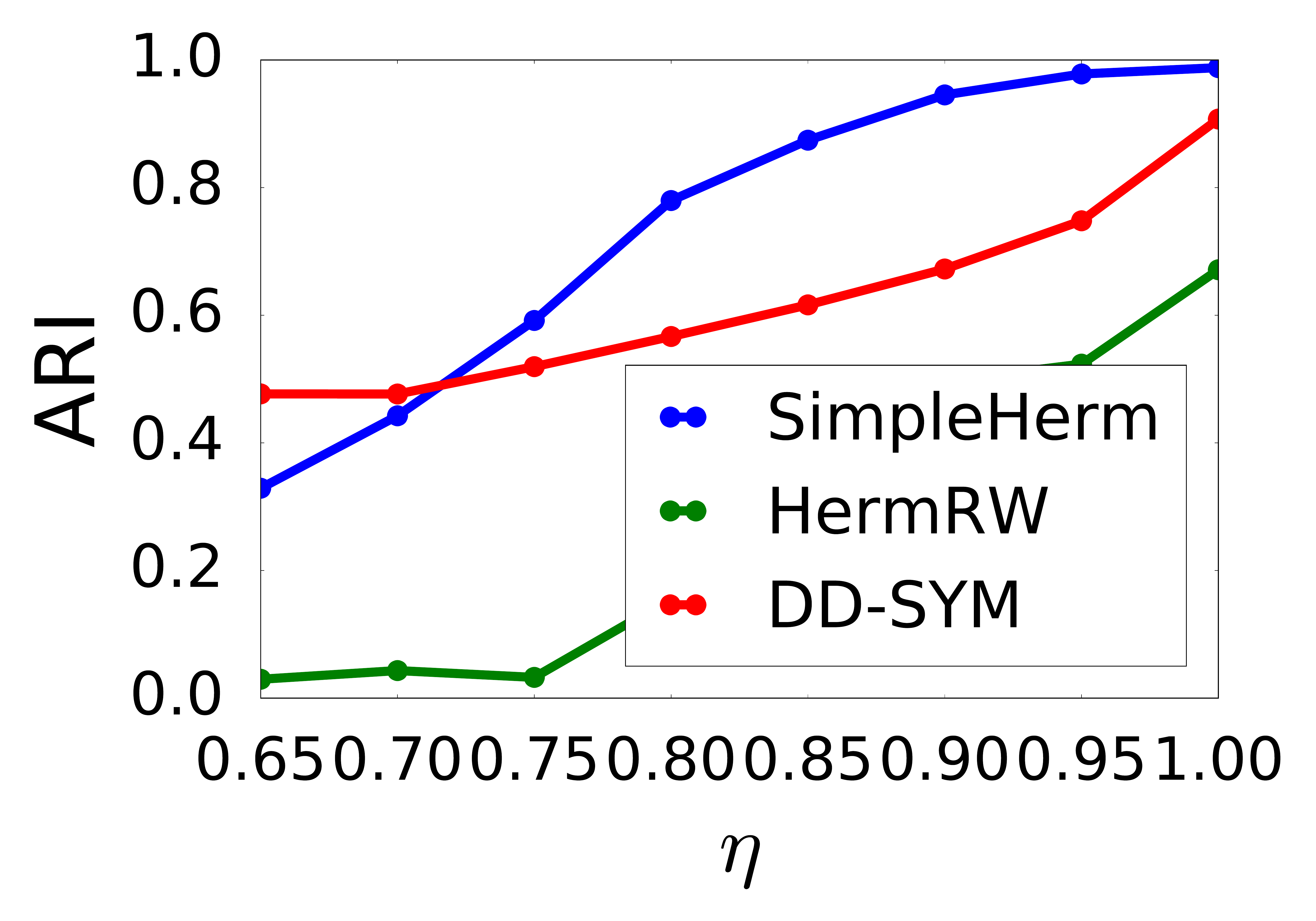}
    \vspace{-0.5cm}
    \caption*{$p=0.05$}
    \end{minipage}%
    \begin{minipage}{0.24\textwidth}
      \centering  
    \includegraphics[width=\textwidth]{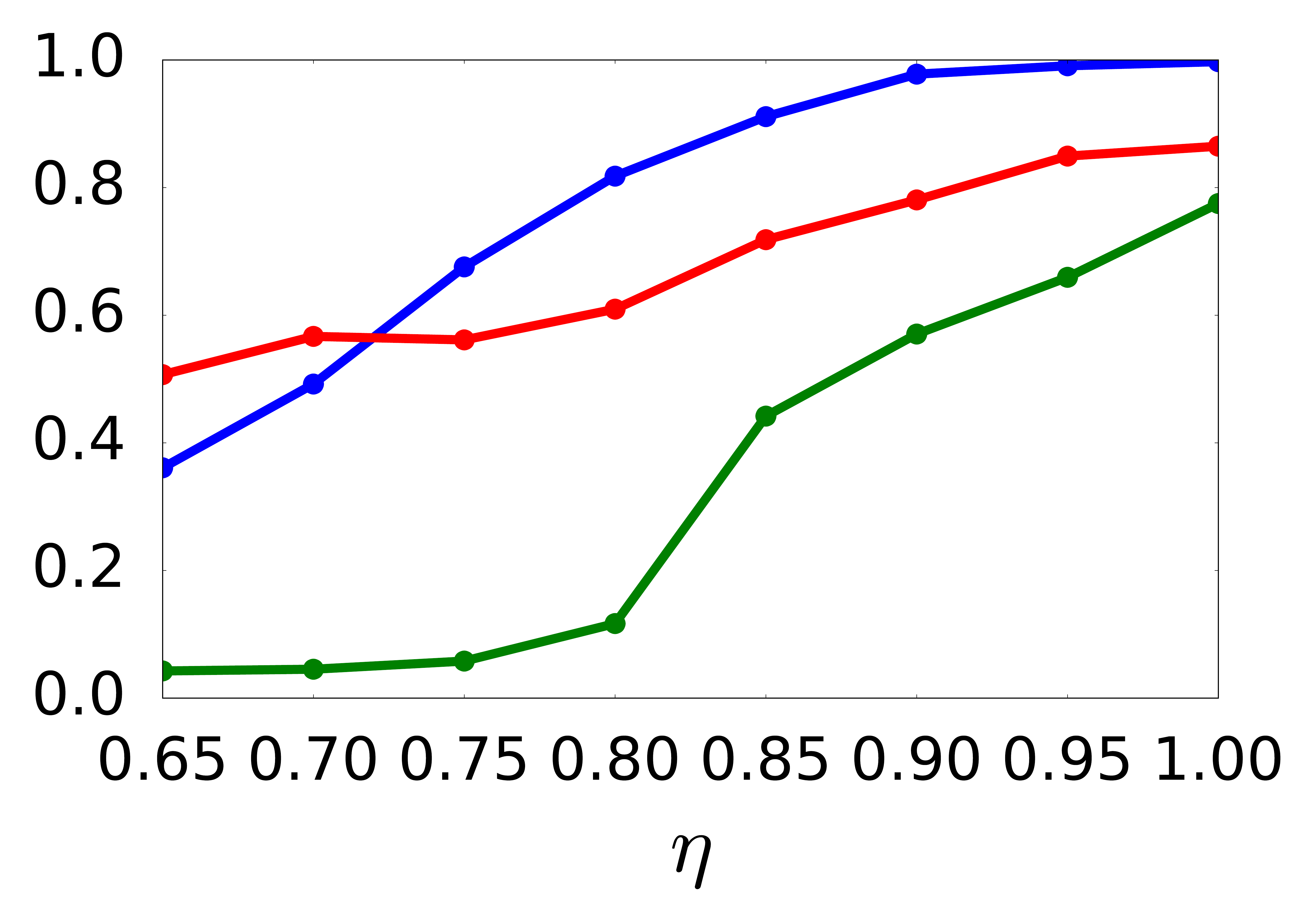}
    \vspace{-0.5cm}
    \caption*{$p=0.06$}
    \end{minipage}
    \begin{minipage}{0.24\textwidth}
      \centering
    \includegraphics[width=\textwidth]{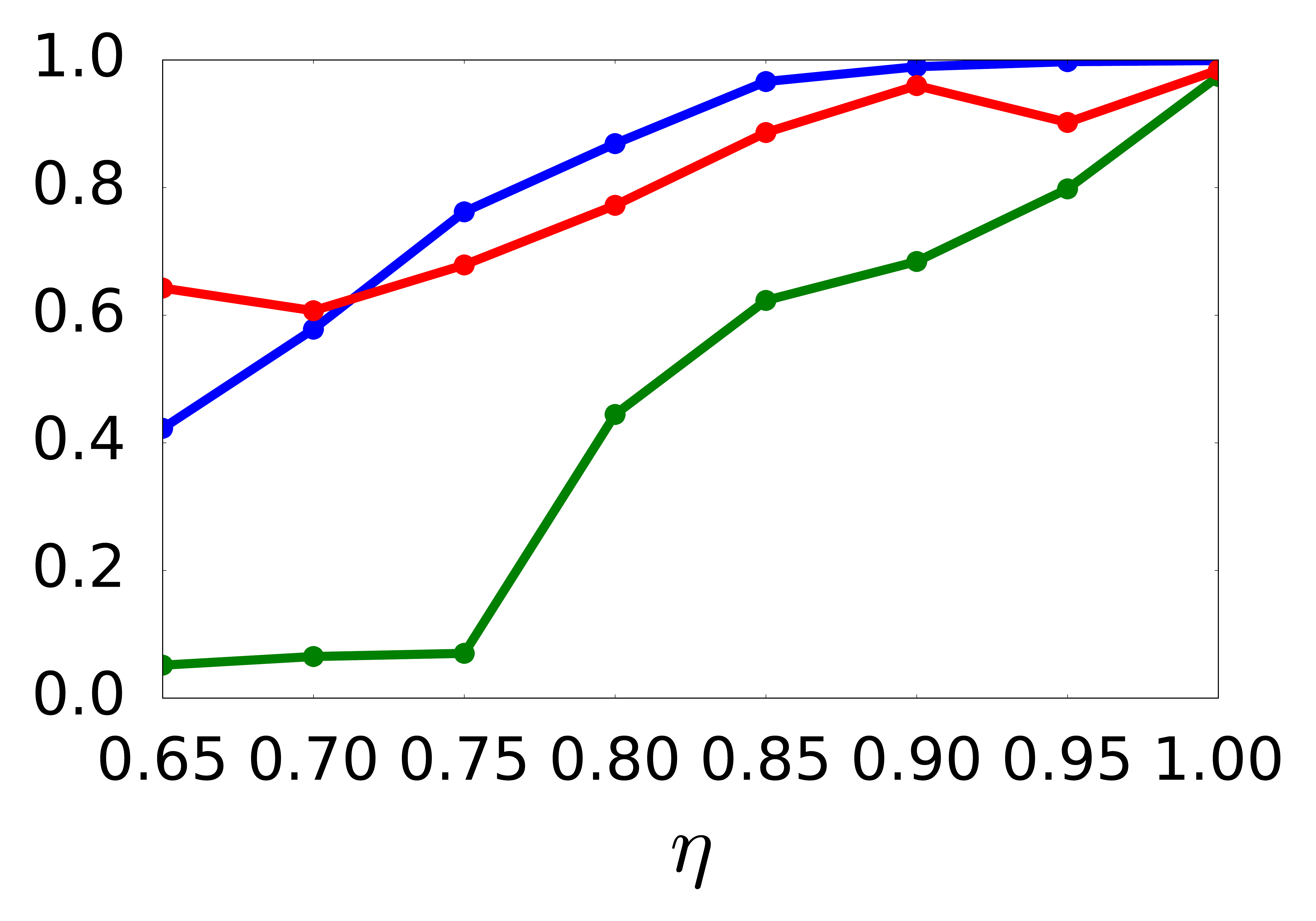}
    \vspace{-0.5cm}
    \caption*{$p=0.075$}
    \end{minipage}%
    \begin{minipage}{0.24\textwidth}
      \centering
    \includegraphics[width=\textwidth]{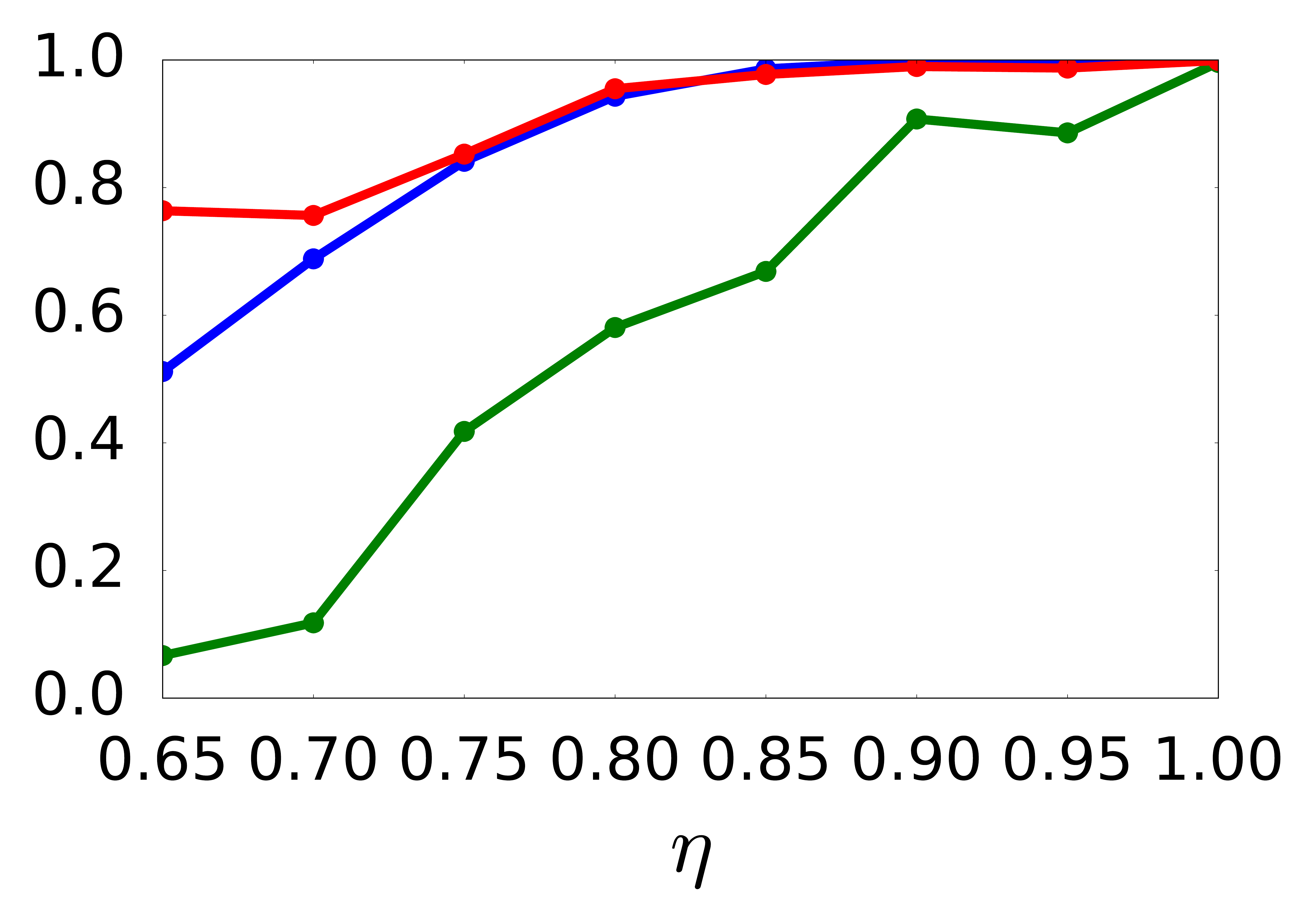}
    \vspace{-0.5cm}
    \caption*{$p=0.1$}
    \end{minipage}%
\caption{\small{$n=2000$ and $k=8$. Average ARIs  over $5$ runs of different algorithms, with respect to different values of $p$ and $\eta$.  }} \label{fig:dsbm_sparse_plot}
\end{figure}

\subsection{Results on
the UN Comtrade Dataset}

We compare our proposed algorithm against the  previous state-of-the-art
on the UN Comtrade Dataset~\cite{screenshot_uncomtrade}. This dataset consists of   the import-export tradeflow data of $97$ specific commodities across $N=246$ countries and territories over the period 1964 -- 2018.  The total size of the data in zipped files is $99.8$GB, where every csv file for a single year contained around 20,000,000 lines.

\paragraph{Pre-processing.} 
As the pre-processing step, for any fixed commodity $c$ and any fixed year,  we construct a directed graph as follows:   the constructed graph has  $N=246$ vertices, which correspond to $246$ countries and territories listed in the dataset. For any two vertices $j$ and $\ell$, there is a directed edge from $j$ to $\ell$ if  the export of commodity $c$ from country $j$ to $\ell$ is greater than the export from $\ell$ to $j$, and the weight of that edge is set to be the absolute value of the difference in trade, i.e., the net trade value between $\ell$ and $j$. Notice that our construction ensures that all the edge weights are non-negative, and there is at most one directed edge between any pair of vertices.  
 
\paragraph{Result on the  International Oil Trade Industry.}
We first study  the international trade   for mineral fuels, oils, and oil distillation product in the dataset.   The primary reason for us to   study the international oil trade   is due to  the fact that crude oil is one of the   highest traded commodities worldwide~\cite{fia_2018}, and plays a significant role in geopolitics~(e.g., 2003 Iraq  War).  Many references in international trade and policy making~(e.g., \cite{behmiri2012crude, cui2015embodied, cui2011design})
allow us to interpret the results of our proposed algorithm.

Following previous studies on the same dataset from  complex networks' perspectives~\cite{du2016spatiotemporal, zhang2018global},  we set $k=4$. Our algorithm's output around the period of 2006--2009 is visualised in Figure~\ref{fig:oil_cluset_change_20062009}. We choose to highlight the results between 2006 and 2009, since 2008 sees the largest post World Ward II oil shock after the economic crisis \cite{kilian2008exogenous}. As discussed earlier, our algorithm's output is naturally associated with an ordering of the clusters that optimises the value of $\Phi$, and this ordering is reflected in our visualisation as well. Notice that such ordering corresponds to the chain of oil trade, and   indicates the clusters of main export countries and import countries for oil trade. 

From Figure~\ref{fig:oil_cluset_change_20062009}, we see that the output of our algorithm from 2006 to 2008 is pretty stable, and this  is in sharp contrast to 
the drastic change between 2008 and 2009, 
caused by the  economic crisis. Moreover, many European countries move across different clusters from 2008 to 2009. The visualisation results of the other algorithms are less significant than ours. 
\begin{figure}[h]
\centering
    \begin{minipage}{0.5\textwidth}
      \centering
    \includegraphics[width=\textwidth]{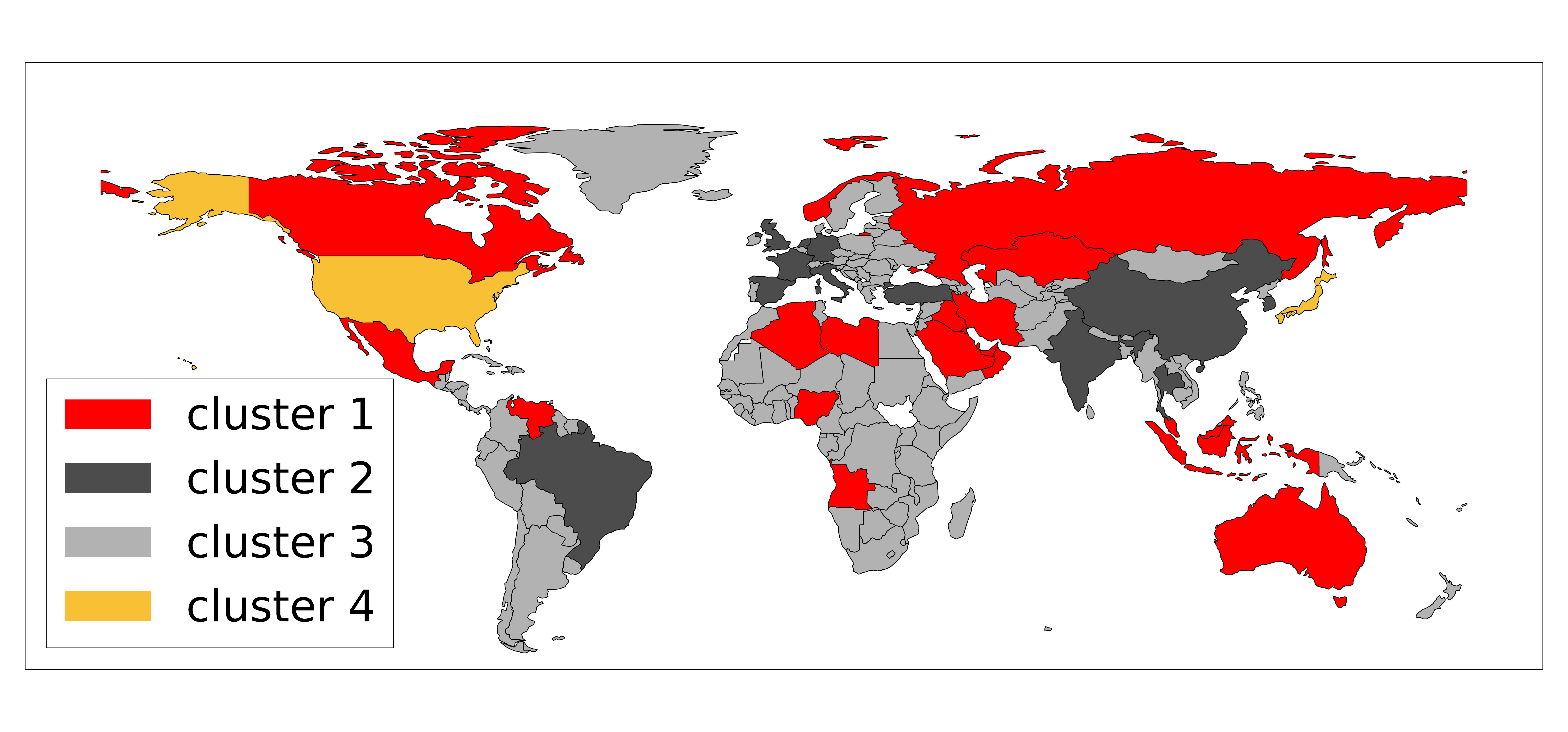}
    \vspace{-0.8cm}
    \caption*{\small{2006}}
    \end{minipage}%
    \begin{minipage}{0.5\textwidth}
      \centering  
    \includegraphics[width=\textwidth]{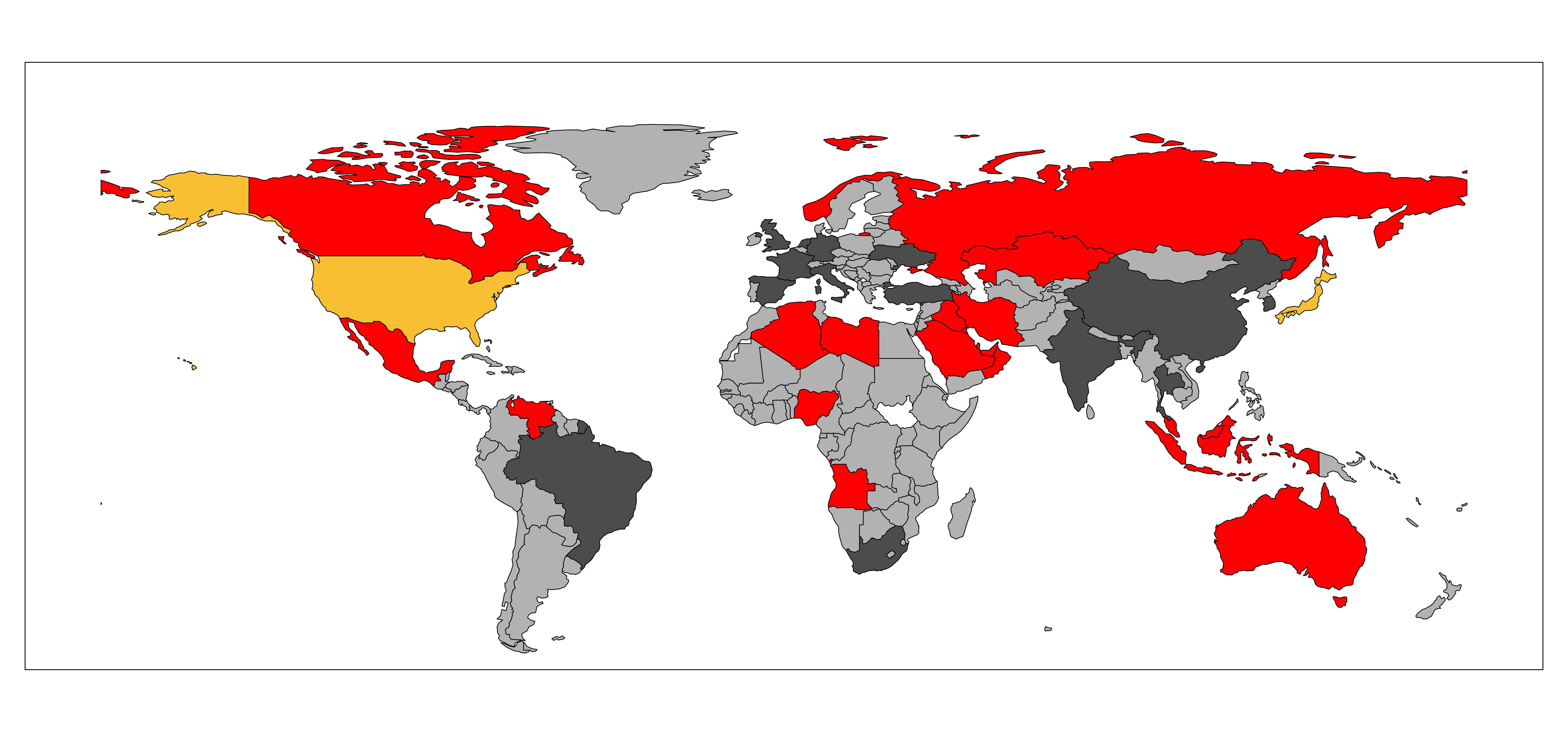}
    \vspace{-0.8cm}
    \caption*{\small{2007}}
    \end{minipage}
    \begin{minipage}{0.5\textwidth}
      \centering
    \includegraphics[width=\textwidth]{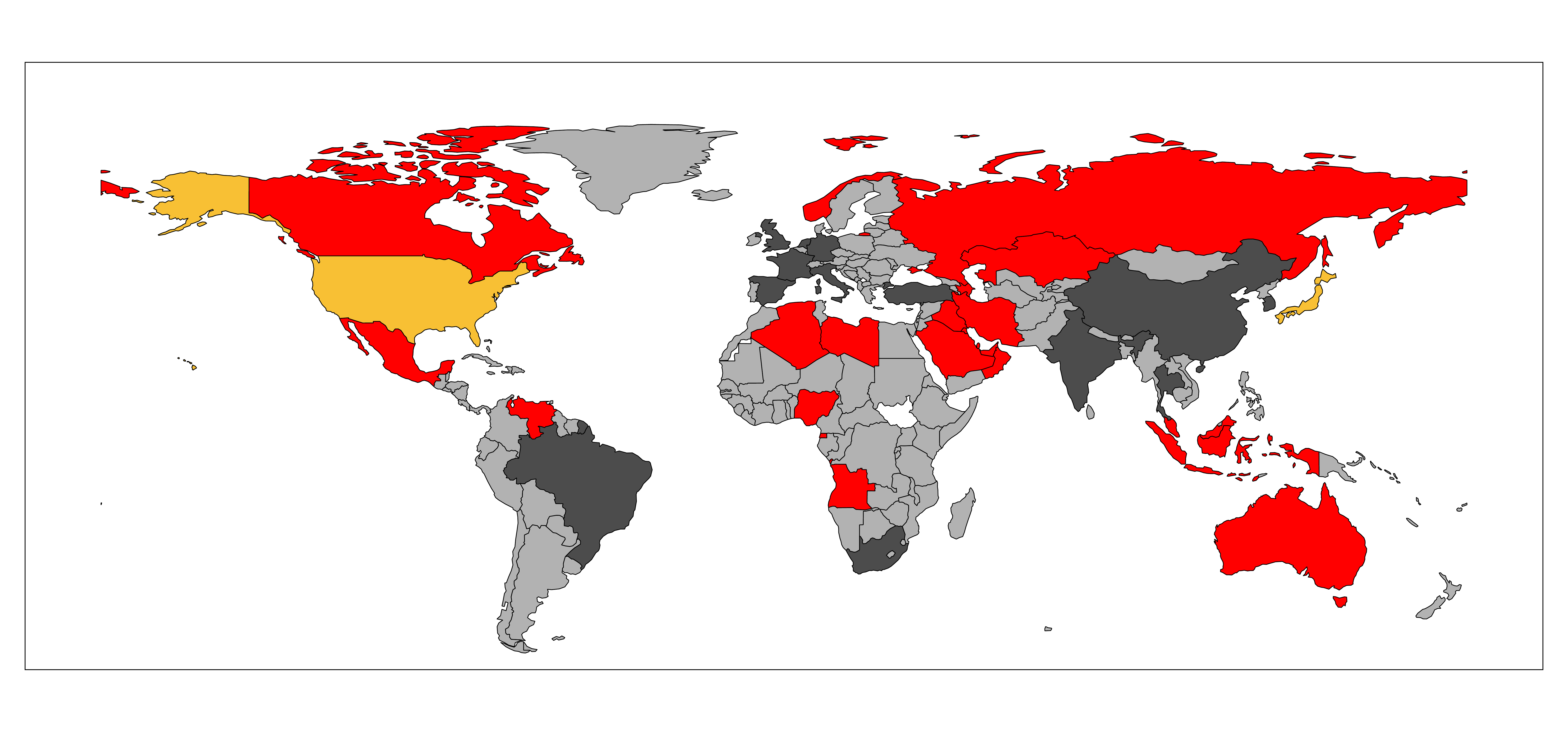}
    \vspace{-0.8cm}
    \caption*{\small{2008}}
    \end{minipage}%
    \begin{minipage}{0.5\textwidth}
      \centering
    \includegraphics[width=\textwidth]{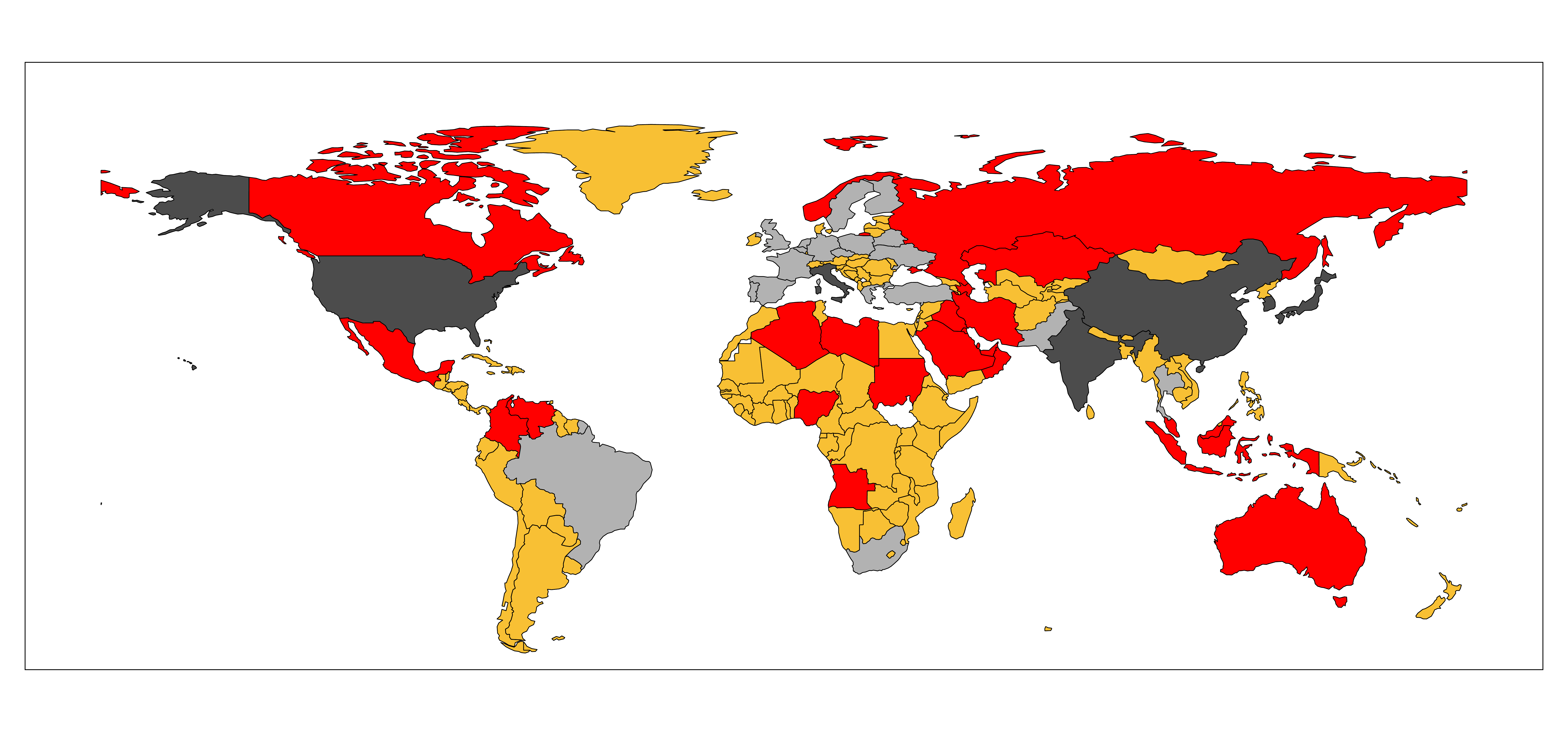}
    \vspace{-0.8cm}
    \caption*{\small{2009}}
    \end{minipage}%
\caption{\small{The clustering result for international trade from 2006 to 2009, where $k=4$. Red countries form start of the trade chain, and yellow countries the end of the trade chain. Countries   coloured white have no data.}} \label{fig:oil_cluset_change_20062009}
\end{figure}

We further show that this  dynamic change of clusters provides a reasonable reflection of international economics. Specifically, we  compute the clustering results of our \texttt{SimpleHerm} algorithm on the same dataset from 2002 to 2017, and compare it with the output of the 
\texttt{DD-SYM} algorithm. For every two consecutive years, we map every cluster to its ``optimal'' correspondence~(i.e., the one that minimises the symmetric difference between the two). 
\begin{figure}
  \begin{center}
    \includegraphics[width=0.7\textwidth]{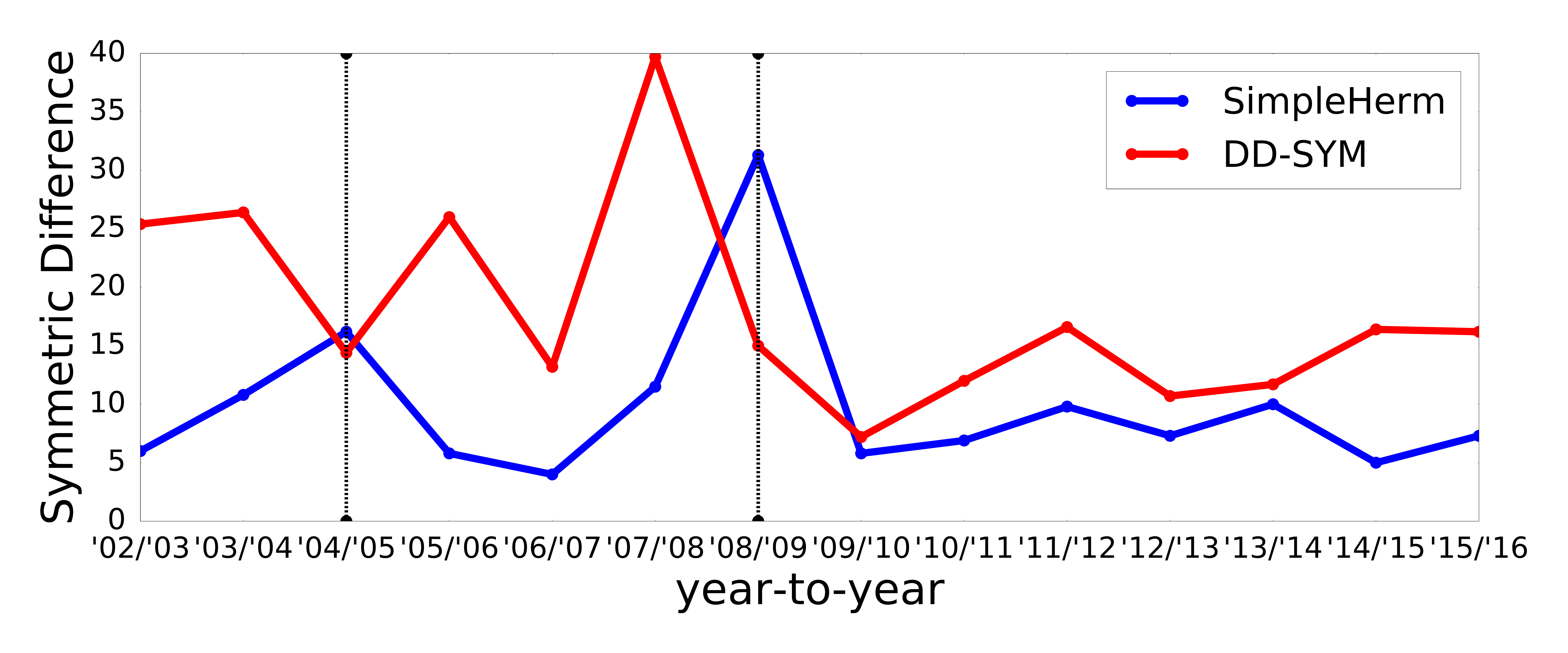}
  \end{center}
  \caption{\small{Comparison of the symmetric difference of the returned clusters between two consecutive years.}\label{fig:symmetric_difference} }
  \end{figure}
We further  compute the total symmetric difference between the clustering results for every  two consecutive years, and our results are visualised in  Figure~\ref{fig:symmetric_difference}.  As shown in   the figure,   our algorithm has notable changes in clustering during 2004/2005 and 2008/2009 respectively. 
The peak around 2004/2005 might be a delayed change as a consequence of the Venezuelan oil strike and the Iraq war of 2003. Both the events led to the decrease in oil barrel production by $5.4$ million barrels per day \cite{hamilton2011historical}. The peak around 2008/2009 is of course due to the economic crisis. These peaks correspond to  the same periods of cluster instability found in the complex network analysis literature~\cite{AN2014254, ZHONG201442}, further signifying our result\footnote{We didn't plot the result between  2016 and 2017, since the symmetric difference for the \texttt{DD-SYM} algorithm is $107$ and the symmetric difference for the  \texttt{SimpleHerm} algorithm is $17$. We believe this is an anomaly for \texttt{DD-SYM}, and plotting this result in the same figure would make it difficult  to compare other years' results.}. Compared to our algorithm, the clustering result of the \texttt{DD-SYM} algorithm is less stable over time.

\paragraph{Result on the International Wood Trade.} 
We also study the international wood trade network~(IWTN). This network looks at the trade of wood and articles of wood. Although the IWTN is less studied than  the
International Oil Trade Industry  in the literature, it is 
nonetheless the reflection of an important and traditional industry and deserves  detailed analysis. Wood trade is dependent on a number of factors, such as the amount of forest a country has left, whether countries are trying to reintroduce forests, and whether countries are deforesting a lot for agriculture (e.g., Amazon rainforest in Brazil)~\cite{kastner2011international}. 

\begin{figure}[h]
\centering
    \begin{minipage}{0.5\textwidth}
      \centering
    \includegraphics[width=\textwidth]{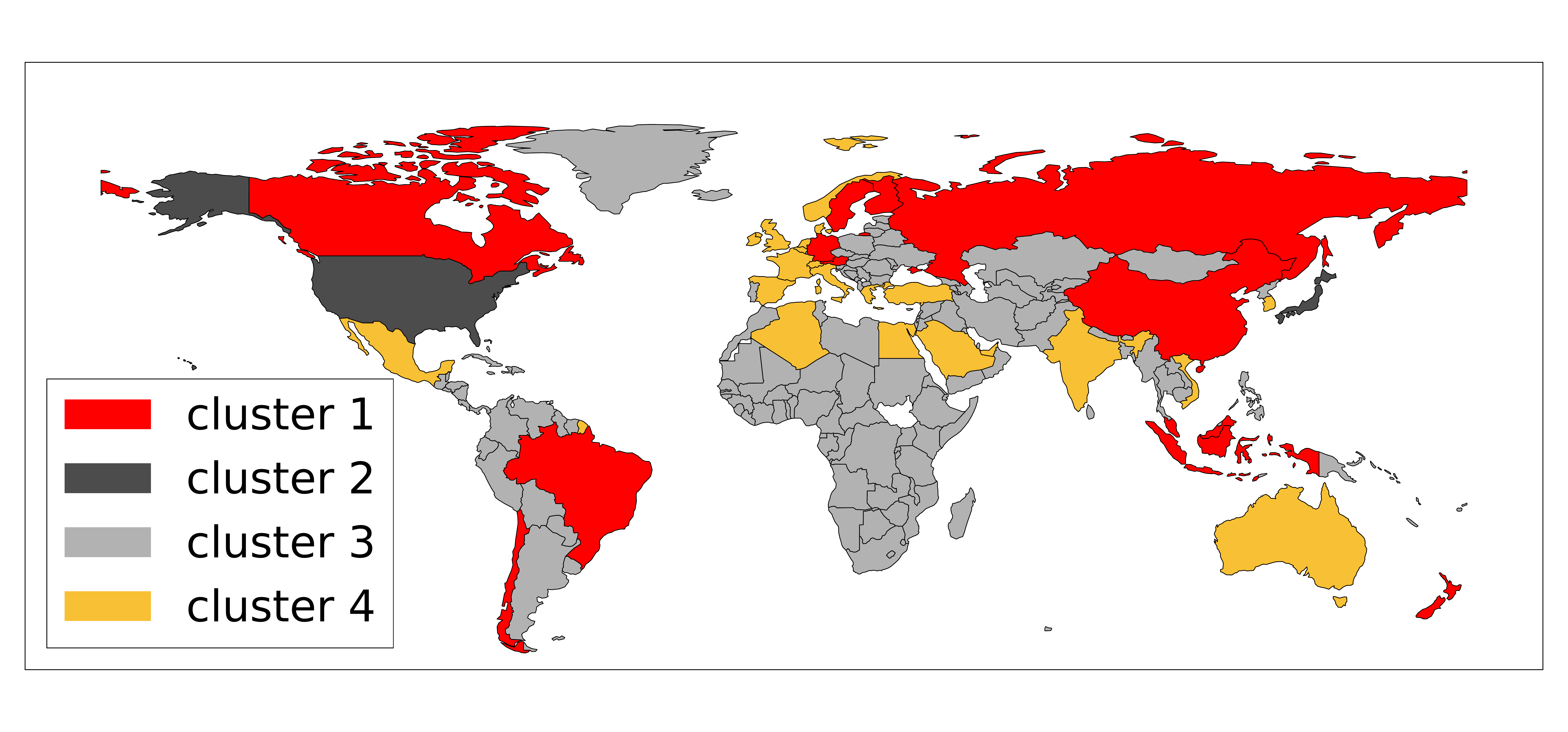}
    \vspace{-0.8cm}
    \caption*{2006}
    \end{minipage}%
    \begin{minipage}{0.5\textwidth}
      \centering  
    \includegraphics[width=\textwidth]{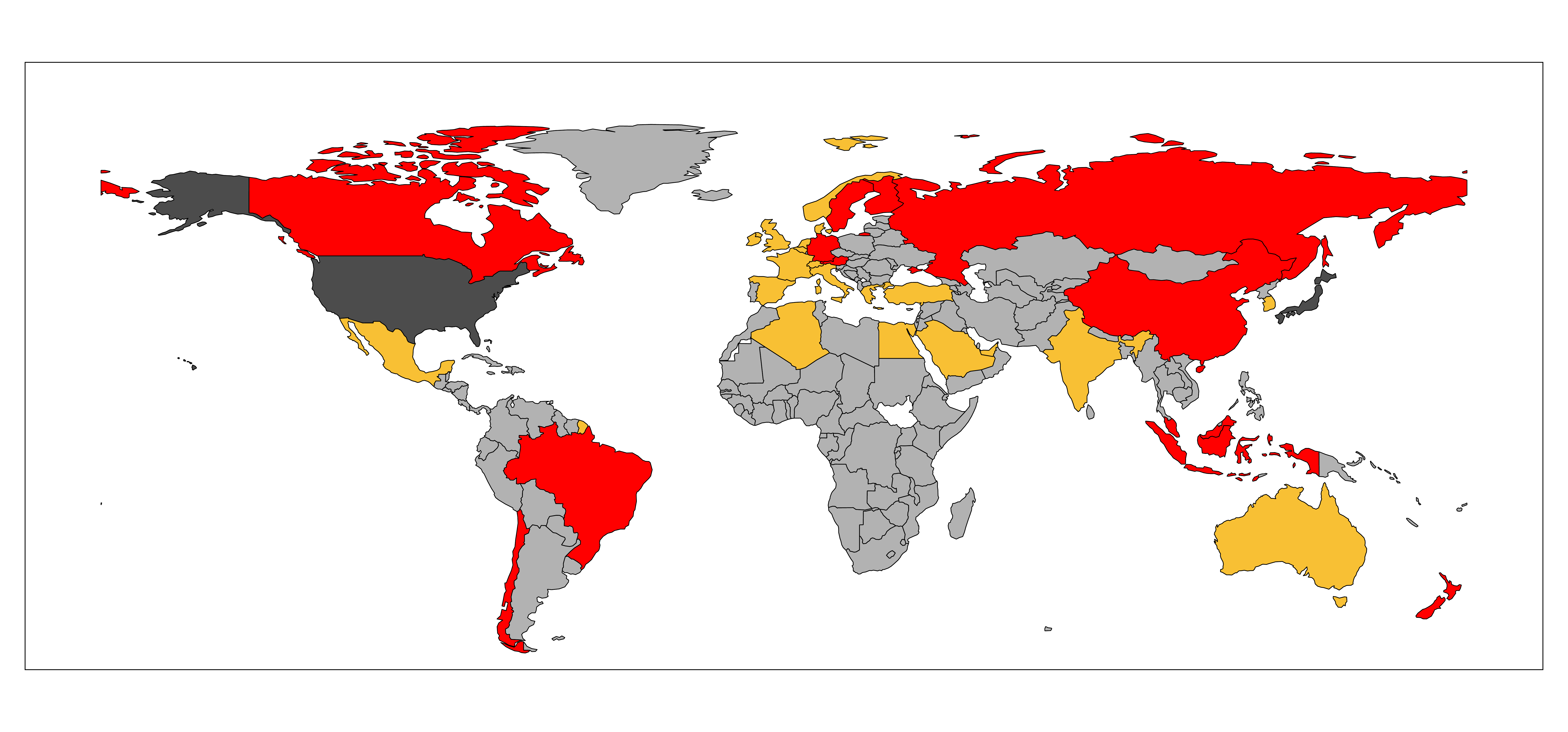}
    \vspace{-0.8cm}
    \caption*{2007}
    \end{minipage}
    \begin{minipage}{0.5\textwidth}
      \centering
    \includegraphics[width=\textwidth]{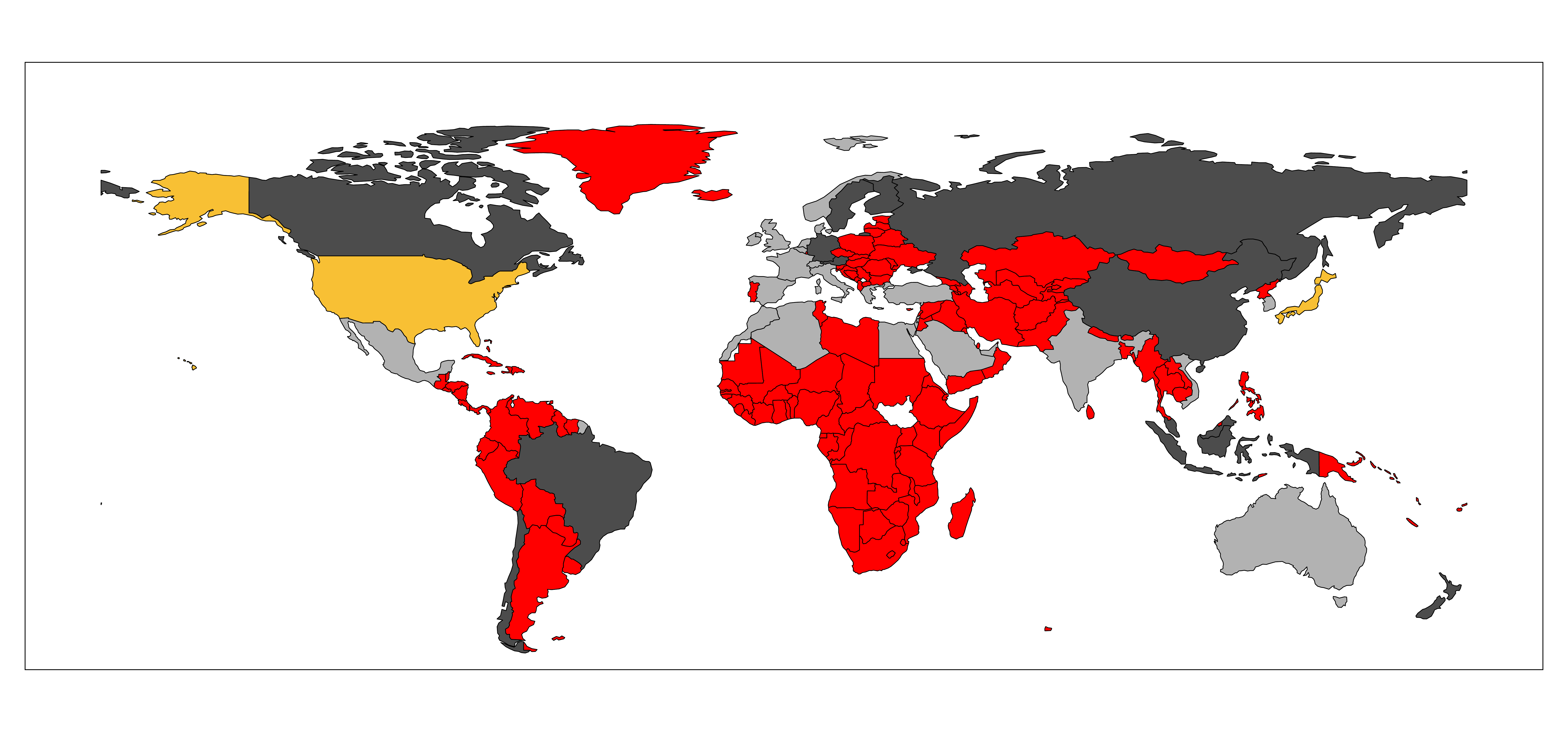}
    \vspace{-0.8cm}
    \caption*{2008}
    \end{minipage}%
    \begin{minipage}{0.5\textwidth}
      \centering
    \includegraphics[width=\textwidth]{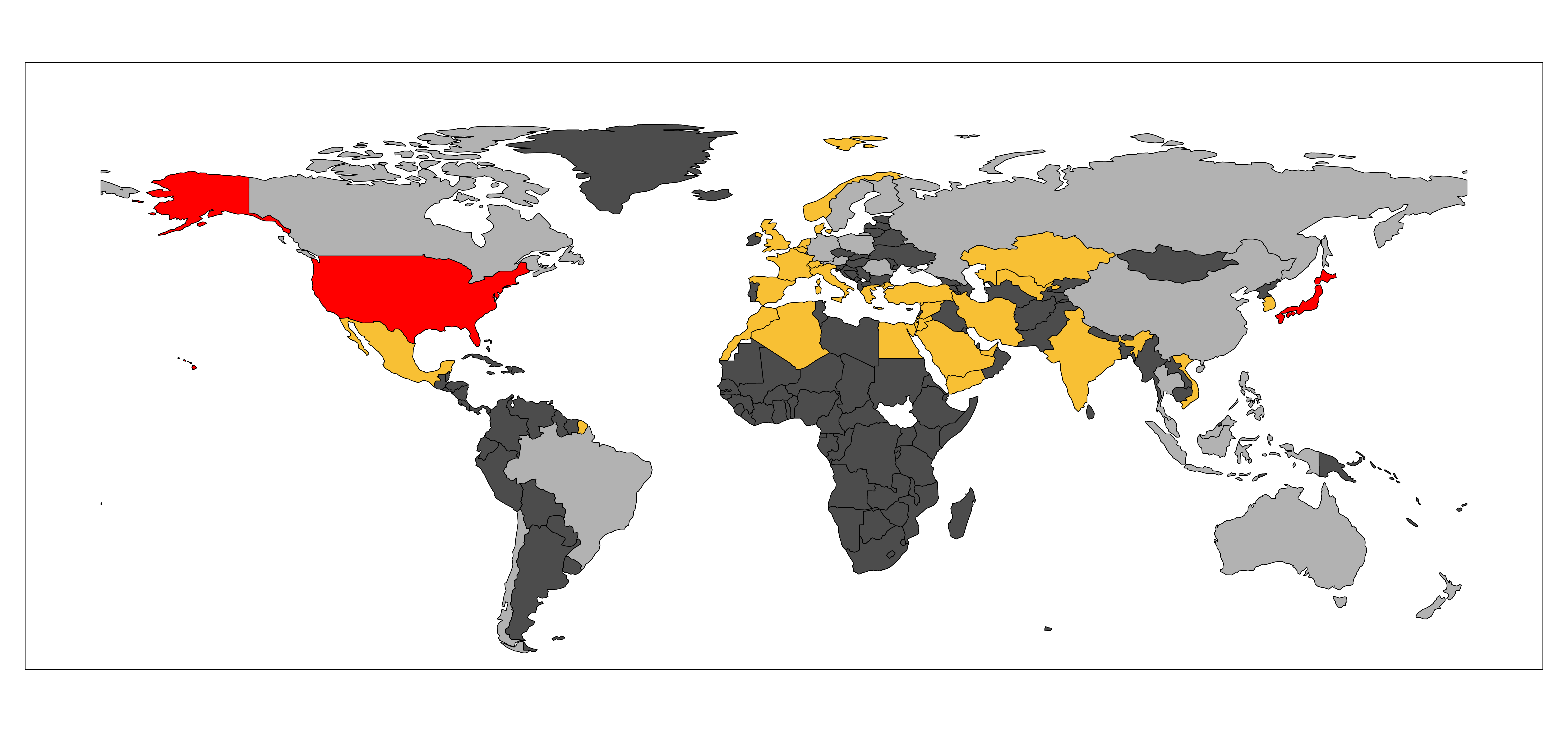}
    \vspace{-0.8cm}
    \caption*{2009}
    \end{minipage}%
\caption{\small{Change in clustering of \texttt{SimpleHerm} of the IWTN from 2006 to 2009 with $k=4$. Clusters are labelled according to their position in the ordering that maximises the flow ratio between the $4$ clusters. Red countries form start of the trade chain, and yellow countries the end of the trade chain. Countries  coloured in  white have no data.}} \label{fig:wood_cluset_change_20062009}
\end{figure}

Figure~\ref{fig:wood_cluset_change_20062009} visualises the clusters from 2006 to 2009. As we can see, the structure of clusters are stable in early years, and the first cluster contains countries with large forests such as Canada, Brazil, Russia, Germany, and China. However, there is a significant change of the cluster structure from 2008 to 2009, and countries in Eastern Europe, the Middle East and Central Asia move across different clusters.


\subsection{Result on  the Data Science for COVID-19 Dataset}

The Data Science for COVID-19 Dataset (DS4C)~\cite{DS4C} contains information about 3519 South Korean citizens  infected with COVID-19. Here, digraphs are  essential to represent how the virus is transmitted among the individuals, and the clusters with high ratio of out-going edges represent the communities worst hit by the virus. We first identify the largest connected component of   the infection graph,   which consists of $67$ vertices and $66$ edges, and run our algorithm on the largest connected component. By setting   $k=4$,
our algorithm   manages to identify a super-spreader as a single cluster, and the path of infection between groups of people along which most infections lie. 

\section{Broader Impact}
The primary focus of our work is  efficient clustering algorithms for digraphs, whose clusters are defined with respect to  the edge directions between different clusters. We believe that our work could have   long-term social impact. For instance,
when modelling the transmission of  COVID-19 among individuals through a digraph, the cluster~(group of people) with the highest ratio of out-going edges represents the most infectious community. This type of information could aid
local containment policy. With the   development of many tracing Apps for COVID-19 and 
a significant amount of infection data available in the near future, our studied algorithm could potentially be applied in this context. In addition, 
as  shown by our   experimental results on    the  UN Comtrade Dataset, our work could be employed to analyse many practical  data for which  most traditional  clustering algorithms do not suffice.

\begin{ack}
Part of this work was done when Steinar Laenen studied at the University of Edinburgh as a Master student. 
He Sun is supported by an EPSRC Early Career Fellowship (EP/T00729X/1).
\end{ack}

{\small
\bibliographystyle{abbrv}

}

\appendix

\section{Omitted details from Section~\ref{sec:encode_flow}
 }\label{sec:appendix_flowstructure}

In this section we present all the technical detailed omitted from Section~\ref{sec:encode_flow}.

\begin{proof}[Proof of Lemma~\ref{lem:boundthetak}]
We prove the statement by analysing the Reyleigh quotient of $\mathcal{L}_G$ with respect to $y$, which is defined by 
$
\frac{y^*\mathcal{L}_G y}{y^* y}.
$
Since $\| y\|=1$, it suffices to analyse $y^* \mathcal{L}_G y$. By definition, we have that 
\begin{align}
y^*\mathcal{L}_G y
& = \frac{1}{k}\left( \sum_{j=0}^{k-1} \widehat{\chi_j} \right)^* \mathcal{L}_G \left( \sum_{j=0}^{k-1} \widehat{\chi_j} \right)  \nonumber\\ 
& =  \frac{1}{k} \left( \sum_{j=0}^{k-1} \frac{D^{1/2} \chi_j}{\|D^{1/2} \chi_j \|}
\right)^* D^{-1/2}  (D-A)D^{-1/2}\left( \sum_{j=0}^{k-1} \frac{D^{1/2} \chi_j}{\|D^{1/2} \chi_j \|} 
\right) \nonumber \\
& =  \frac{1}{k} \left( \sum_{j=0}^{k-1} \frac{ \chi_j}{\|D^{1/2} \chi_j \|}
\right)^* \left( D - A\right)  \left( \sum_{j=0}^{k-1} \frac{  \chi_j}{\|D^{1/2} \chi_j \|} 
\right).  \label{eq:qf}
\end{align}
To   analyse  \eq{qf}, first of all it is easy to see that  
\begin{align}
 \frac{1}{k} \left( \sum_{j=0}^{k-1} \frac{ \chi_j}{\|D^{1/2} \chi_j \|}
\right)^*   D  \left( \sum_{j=0}^{k-1} \frac{  \chi_j}{\|D^{1/2} \chi_j \|} 
\right) = &  \frac{1}{k} \left( \sum_{j=0}^{k-1} \frac{ \chi_j}{\|D^{1/2} \chi_j \|}
\right)^*   D^{1/2} D^{1/2}  \left( \sum_{j=0}^{k-1} \frac{  \chi_j}{\|D^{1/2} \chi_j \|} 
\right)\notag\\
= &  \frac{1}{k} \left( \sum_{j=0}^{k-1} \frac{ D^{1/2}\chi_j}{\|D^{1/2} \chi_j \|}
\right)^*     \left( \sum_{j=0}^{k-1} \frac{  D^{1/2} \chi_j}{\|D^{1/2} \chi_j \|} 
\right)\notag\\
= &  \frac{1}{k}  \sum_{j=0}^{k-1} \frac{\chi_j^*D\chi_j}{\|D^{1/2} \chi_j \|^2}\notag\\
= & 1, \label{eq:sum1}
\end{align}
where the third equality follows by the fact that $\chi_j^*\chi_{\ell}=0$ for any $0\leq j\neq \ell\leq k-1$.  On the other hand,  by definition we have that 
\begin{align}
  \lefteqn{ \frac{1}{k} \left( \sum_{j=0}^{k-1} \frac{ \chi_j}{\|D^{1/2} \chi_j \|}
\right)^*     A  \left( \sum_{j=0}^{k-1} \frac{  \chi_j}{\|D^{1/2} \chi_j \|}
\right) } \notag \\
& = \frac{1}{k} \left( \sum_{j=0}^{k-1} \frac{ \chi_j}{\sqrt{\vol\left(S_j\right)}}
\right)^*     A  \left( \sum_{j=0}^{k-1} \frac{  \chi_j}{\sqrt{\vol\left(S_j\right)}}
\right) \notag\\
& = \frac{1}{k}\cdot \sum_{j=0}^{k-1} \sum_{\ell=0}^{k-1}  \sum_{ \substack{ u\leadsto v \\ u\in S_j, v\in S_{\ell} }} \left( \frac{\overline{\chi_j}(u)}{\sqrt{\vol\left(S_j\right)}}\cdot A_{u,v}\cdot \frac{\chi_{\ell}(v)}{ \sqrt{\vol\left( S_{\ell}\right)}} + \frac{\overline{\chi_{\ell}}(v)}{\sqrt{\vol\left(S_{\ell}\right)}}\cdot A_{v,u}\cdot \frac{\chi_{j}(u)}{ \sqrt{\vol\left( S_{j}\right)}} \right)  \notag \\
& = \frac{1}{k}\cdot\sum_{j=0}^{k-1} \sum_{\ell=0}^{k-1} \sum_{ \substack{ u\leadsto v \\ u\in S_j, v\in S_{\ell} }} \frac{w(u,v)}{\sqrt{\vol\left(S_{j}\right)}\cdot \sqrt{\vol\left(S_{\ell}\right)} }\cdot 2\cdot \mathsf{Re}\left( \left(\unit\right)^{\ell+1-j} \right) \notag\\
& = \frac{1}{k}\cdot \sum_{j=0}^{k-1}\sum_{\ell=0}^{k-1}\frac{w\left(S_j, S_{\ell} \right)}{\sqrt{\vol\left(S_{j}\right)}\cdot \sqrt{\vol\left(S_{\ell}\right)} }\cdot 2\cdot \cos\left(  \frac{ 2\pi\cdot (\ell+1-j)}{\lceil 2\pi\cdot k\rceil} \right) \label{eq:sep},
\end{align}
where   $\mathsf{Re}(\cdot)$ stands for the real part of a complex number.   Combining \eq{qf}, \eq{sum1} with \eq{sep}, we have that 
\begin{align*}
\lefteqn{ y^*\mathcal{L}_G y }\\
& = 1 - \frac{1}{k}\cdot \sum_{j=0}^{k-1}\sum_{\ell=0}^{k-1}\frac{w\left(S_j, S_{\ell} \right)}{\sqrt{\vol\left(S_{j}\right)}\cdot \sqrt{\vol\left(S_{\ell}\right)} }\cdot 2\cdot \cos\left(  \frac{ 2\pi\cdot (\ell+1-j)}{\lceil 2\pi\cdot k\rceil} \right)\\
& \leq  1 - \frac{1}{k} \cdot \sum_{j=0}^{k-1}\sum_{\ell=0}^{k-1} \frac{w\left(S_j, S_{\ell} \right)}{\sqrt{\vol(S_j)} \sqrt{\vol(S_{\ell})} }\cdot \left( 2-  \left(  \frac{ 2\pi\cdot (\ell+1-j)}{\lceil 2\pi\cdot k\rceil} \right)^2\right)  \\  
& \leq 1 - \frac{1}{k}\cdot \sum_{j=0}^{k-1}\sum_{\ell=0}^{k-1} \frac{2\cdot w\left(S_j, S_{\ell} \right)}{\sqrt{\vol(S_j)} \sqrt{\vol(S_{\ell})}  } +\frac{1}{k}\cdot \sum_{j=0}^{k-1}\sum_{\ell=0}^{k-1}  \frac{ w\left(S_j, S_{\ell} \right)}{\sqrt{\vol(S_j)} \sqrt{\vol(S_{\ell})}  } \left(  \frac{    \ell+1-j }{    k } \right)^2 \\
& \leq 1 - \frac{1}{k}\cdot \sum_{j=0}^{k-1}\sum_{\ell=0}^{k-1} \frac{2\cdot w\left(S_j, S_{\ell} \right)}{\sqrt{\vol(S_j)} \sqrt{\vol(S_{\ell})}  } +\frac{1}{k}\cdot \sum_{j=0}^{k-1}\sum_{\substack{ 0\leq \ell\leq k-1 \\ \ell\neq j-1 } }  \frac{ 2\cdot w\left(S_j, S_{\ell} \right)}{\sqrt{\vol(S_j)} \sqrt{\vol(S_{\ell})}  } \left(  \frac{    \ell+1-j }{    k } \right)^2
\\
& = 1 -  \frac{1}{k}\cdot \sum_{j=0}^{k-1}\sum_{\substack{ 0\leq \ell\leq k-1 \\ \ell\neq j-1 } }\frac{2 \cdot w(S_j, S_{\ell})}{\sqrt{\vol(S_j)} \sqrt{\vol(S_{\ell})} } \left(1 - \left(  \frac{    \ell+1-j }{    k } \right)^2\right) - \frac{1}{k}\cdot \sum_{j=1}^{k-1}\frac{2 \cdot w(S_j, S_{j-1})}{\sqrt{\vol(S_j)} \sqrt{\vol(S_{j-1})} } \\
& \leq 1 - \frac{1}{k}\cdot \sum_{j=1}^{k-1}\frac{2 \cdot w(S_j, S_{j-1})}{\sqrt{\vol(S_j)} \sqrt{\vol(S_{j-1})} } \\
& = 1 - \frac{2}{k}\cdot \sum_{j=1}^{k-1}\frac{w(S_j, S_{j-1})}{\sqrt{\vol(S_j)} \sqrt{\vol(S_{j-1})} } \\
& \leq 1- \frac{4}{k}\cdot \sum_{j=1}^{k-1} \frac{w(S_j, S_{j-1})}{\vol(S_j) + \vol(S_{j-1})}\\
& = 1 -  \frac{4}{k}\cdot\theta_k(G),
\end{align*} 
where the first inequality follows by the fact that $\cos x \geq 1- x^2/2$ and the last inequality follows by the inequality $2ab\leq a^2 + b^2$ for any $a,b\in\mathbb{R}$. Therefore,  
we have  that 
 \[
 \frac{y^* \mathcal{L}_G y}{ y^* y}
 \leq  1- \frac{4}{k}\cdot\theta_k(G).
 \]
By the Rayleigh characterisation of eigenvalues we know that 
\begin{equation*}
\lambda_1(\mathcal{L}_{G}) = \min_{x \in \mathbb{C}^{n} \setminus \{0\}}  \frac{x^{*}\mathcal{L}_{G}x}{x^{*}x}  \leq  1- \frac{4}{k}\cdot\theta_k(G), 
\end{equation*} 
which proves the first statement of the lemma.
 
Now we prove the second statement. 
Let $G$ be a digraph, and $S_0,\ldots, S_{k-1}$ be the $k$ clusters maximising $\Phi_G(S_0,\ldots, S_{k-1})$, i.e., $\Phi_G(S_0,\ldots, S_{k-1})=\theta_k(G)$. Since adding edges that are not along the path  only decreases the value of $\Phi_G$, we assume without loss of generality that all the edges are along the path.  For the base case of $k=2$, we have that 
\[
\Phi_G(S_0, S_1) = \frac{w(S_0, S_1)}{ \vol(S_0) + \vol (S_1)} = \frac{1}{2} =\frac{k}{4}. 
\]
Next, we will prove that   $\theta_k(G)< k/4$ for any $k\geq 3$.
 We set   $y_j\triangleq w(S_j, S_{j-1})$ for any $1\leq j\leq k-1$, and   have that
\begin{align*}
\Phi_G(S_0,\ldots, S_{k-1}) &  = \sum_{j=1}^{k-1} \frac{w(S_j, S_{j-1})}{ \vol(S_j) + \vol(S_{j-1}) } \\
&=\frac{y_1}{2y_1 + y_2} +  \sum_{j=2}^{k-2} \frac{y_j}{y_{j-1} + 2y_j + y_{j+1}}  + \frac{y_{k-1}}{ y_{k-2} + 2y_{k-1}}. 
\end{align*}
By introducing $y_0=0$ and assuming that all the indices of $\{y_j\}_j$ are modulo b $k$, we can write $\Phi_G(S_0,\ldots, S_{k-1})$ as 
\[
\Phi_G(S_0,\ldots, S_{k-1}) = \sum_{j=0}^{k-1}\frac{y_j}{y_{j-1} + 2y_j + y_{j+1}}.
\]
Next we compute $\partial \Phi_G / \partial y_j$, and have that
\begin{align*}
     \frac{\partial \Phi_G}{\partial y_j} & = \frac{\partial \Phi_G}{\partial y_j} \sum_{j=0}^{k-1}\frac{y_j}{y_{j-1} + 2y_j + y_{j+1}} \\
      & = \frac{\partial \Phi_G}{\partial y_j} \Bigg ( \frac{y_{j-1}}{y_{j-2} + 2y_{j-1} + y_{j}} + \frac{y_j}{y_{j-1} + 2y_j+ y_{j+1}} +   \frac{y_{j+1}}{y_{j} + 2y_{j+1} + y_{j+2}} \Bigg) \\
      & = - \frac{y_{j-1}}{\left(y_{j-2} + 2y_{j-1} + y_{j}\right)^2} + \frac{y_{j-1} + y_{j+1}}{\left(y_{j-1} + 2y_j+ y_{j+1}\right)^2}   - \frac{y_{j+1}}{\left(y_{j} + 2y_{j+1} + y_{j+2}\right)^2}.
\end{align*}
Notice that, when all the $y_j (0\leq j\leq k-1)$ equal to the same non-zero value, it holds that $\partial \Phi_G / \partial y_j=0$ for any $j$, and $\theta_G(S_0,\ldots, S_{k-1}) = k/4$.
Moreover, it's easy to verify that $k/4$
is an upper bound of  $\theta_G$.  
Since we effectively assume that $y_0=0$, which cannot be always equal to all of the $y_1,\ldots, y_{k-1}$, we have that $\theta_G(S_0,\ldots, S_{k-1})<k/4$.
 \end{proof}

\begin{proof}[Proof of Theorem~\ref{thm:structure_theorem}]  
 We first prove the first statement. 
We  write $y$ as a linear combination of the eigenvectors of $\mathcal{L}_G$ by
\begin{equation*}
y = \alpha_1 f_1 + \cdots + \alpha_n f_n
\end{equation*}
for some  $\alpha_i \in \mathbb{C}$ and $f_i \in \mathbb{C}^{n}$, and define $\tilde{f_1}$ by
$
    \tilde{f_1} \triangleq \alpha_1 f_1.
$
By the definition of the Rayleigh quotient for Hermitian matrices we have that
\begin{align*}
    \frac{y^* \mathcal{L}_G y}{\| y\|} &= (\alpha_{1} f_{1} + \cdots + \alpha_{n} f_{n})^{*} \mathcal{L}_{G} (\alpha_{1} f_{1} + \cdots + \alpha_{n} f_{n}) \\
    & = \norm{\alpha_1}^2\lambda_1(\mathcal{L}_{G}) + \cdots + \norm{\alpha_n}^2\lambda_n(\mathcal{L}_{G}) \\
    & \geq \norm{\alpha_1}^2\lambda_1(\mathcal{L}_{G}) + (\norm{\alpha_2}^2 + \cdots + \norm{\alpha_n}^2)\lambda_2(\mathcal{L}_{G})\\
    & \geq (1 - \norm{\alpha_1}^2)\lambda_2(\mathcal{L}_{G}), 
\end{align*}
where the first inequality holds by the fact that $\lambda_1(\mathcal{L}_{G})\leq \ldots\leq \lambda_n(\mathcal{L}_{G})$  and  the second inequality holds by the fact that   $\norm{\alpha_2}^2 + \cdots + \norm{\alpha_n}^2 = 1 - \norm{\alpha_1}^2$. We can see that 
\begin{equation*}
    \norm{y - \tilde{f_1}}^2 = \norm{\alpha_2}^2 + \cdots + \norm{\alpha_n}^2 = 1 - \norm{\alpha_1}^2 \leq \frac{1}{\lambda_2} \cdot \frac{y^* \mathcal{L}_G y}{\| y\|} \leq   \frac{1}{\gamma_k(G)}.
\end{equation*}
Setting $\alpha=\alpha_1$ proves the first statement.

 Next we prove the second statement.  By the relationship between $f_1$ and $\tilde{f_1}$, we write
\[
f_1 = \beta_1 \tilde{f_1},
\]
where $\beta_1\triangleq 1/\alpha_1$ is the multiplicative inverse of $\alpha_1$. Then, we define $\tilde{y}$ as 
\[
 \tilde{y} = \beta_1 y =   \beta_1\left(\alpha_1 f_1 + \cdots + \alpha_n f_n\right) = f_1 + \beta_1(\alpha_2 f_2 + \cdots + \alpha_n f_n),
\]
and this implies that
\begin{align}
    \norm{f_1 - \tilde{y}}^2 &= \norm{\beta_1\left( \alpha_2f_2 + \cdots + \alpha_n f_n \right)}^2=  \overline{\beta_1}\cdot \left( \sum_{j=2}^n \norm{\alpha_j}^2 \right) \cdot \beta_1  = \frac{1}{\norm{\alpha_1}^2}\left( 1 - \norm{\alpha_1}^2 \right)  \nonumber \\
    & \leq \frac{1}{\norm{\alpha_1}^2 \cdot  \gamma_k(G)} \label{eq:midstp}.
\end{align}
Since $1-\| \alpha_1\|^2\leq 1/\gamma_k(G)$ implies that
\[
\|\alpha_1 \|^2 \geq \frac{ \gamma_k(G)-1}{ \gamma_k(G)},
\]
we can rewrite \eq{midstp} as 
\[
 \norm{f_1 - \tilde{y}}^2 \leq \frac{1}{ \gamma_k(G)-1},
\]
and therefore setting $\beta=\beta_1$ 
proves the second statement.
\end{proof}

 

\section{Omitted details from Section~\ref{sec:Algorithm}
 }\label{sec:appendix_algorithms}

In this section we present all the technical detailed omitted from Section~\ref{sec:Algorithm}.

\begin{proof}[Proof of Lemma~\ref{lem:distance_p_x}]
 
By definition, we have that 
\begin{align*}
    \sum_{j=0}^{k-1}\sum_{u\in S_j}
     d_u\cdot  \left\| F(u) - p^{(j)} \right\|^2 & = 
      \sum_{j=0}^{k-1}\sum_{u\in S_j}
     d_u\cdot \left\| \frac{1}{\sqrt{d_u}}\cdot f_1(u) - \frac{\beta}{\sqrt{k}}\cdot \frac{  (\omega_{\lceil 2\pi\cdot k\rceil})^j}{\sqrt{\vol(S_j)}} \right\|^2 \\
     & = \sum_{j=0}^{k-1}\sum_{u\in S_j}
      \left\| f_1(u) - \sqrt{\frac{d_u}{k\cdot \vol(S_j)}} \cdot\beta\cdot (\omega_{\lceil 2\pi\cdot k\rceil})^j  \right\|^2 \\
      & = \sum_{j=0}^{k-1}\sum_{u\in S_j} \left\| f_1(u) - \tilde{y}(u) \right\|^2\\
      & = \left\| f_1 - \tilde{y} \right\|^2 \\
      & \leq \frac{1}{\gamma_k(G)-1},
\end{align*}
where the last inequality follows by Theorem~\ref{thm:structure_theorem}.
\end{proof}

\begin{proof}[Proof of Lemma~\ref{lem:pi}]
The proof is by direct calculation on $\left\|p^{(j)}\right\|^2$.
\end{proof}

\begin{proof}[Proof of Lemma~\ref{lem:distancepjpl}]
By definition of $p^{(j)}$ and $p^{(\ell)}$, we have that
\begin{align}
\lefteqn{ \left\| p^{(j)} -p^{(\ell)} \right\|^2} \notag \\
& = \left\| p^{(j)} \right\|^2 + \left\| p^{(\ell)} \right\|^2 - 2\cdot\mathrm{Re}\left(p^{(j)}\cdot\overline{p^{(\ell)}}\right) \notag \\
& = \frac{ \|\beta \|^2}{k\cdot \vol(S_j)} +
\frac{ \|\beta \|^2}{k\cdot \vol(S_{\ell})} - 2\cdot\mathrm{Re}\left(\frac{\beta\cdot \left(\omega_{\lceil 2\pi\cdot k\rceil} \right)^j}{\sqrt{k}\cdot\sqrt{\vol(S_j)}} \cdot\frac{\overline{\beta}\cdot \left(\omega_{\lceil 2\pi\cdot k\rceil} \right)^{-\ell}}{\sqrt{k}\cdot\sqrt{\vol(S_{\ell})}} \right) \notag\\
& = \frac{ \|\beta \|^2}{k\cdot \vol(S_j)} +
\frac{ \|\beta \|^2}{k\cdot \vol(S_{\ell})} - 2\cdot\frac{\|\beta\|^2}{k\cdot \sqrt{\vol(S_j)\cdot \vol(S_{\ell})}}\cdot\cos\left(\frac{2\pi\cdot(j-\ell)}{\lceil 2\pi\cdot k\rceil}\right) \label{eq:replaceangle}.
\end{align}
For the case of calculation and the fact that $\cos(x) = \cos(-x)$ for any $x\in\mathbb{R}$, we denote
\[
\eta \triangleq \frac{2\pi\cdot |j-\ell| }{\lceil 2\pi\cdot k\rceil},
\]
and rewrite \eq{replaceangle} as
\begin{align*}
\lefteqn{ \left\| p^{(j)} -p^{(\ell)} \right\|^2} \notag \\
& =\frac{ \|\beta \|^2}{k\cdot \vol(S_j)} +
\frac{ \|\beta \|^2}{k\cdot \vol(S_{\ell})} - 2\cdot\frac{\|\beta\|^2}{k\cdot \sqrt{\vol(S_j)\cdot \vol(S_{\ell})}}\cdot\cos\eta  \\
& = \frac{ \|\beta \|^2}{k\cdot \max\{\vol(S_j),\vol(S_{\ell}) \}} +
\frac{ \|\beta \|^2\cdot (\sin^2\eta + \cos^2\eta)}{k\cdot \min\{\vol(S_j),\vol(S_{\ell}) \}} -  \frac{2\cos\eta\cdot\|\beta\|^2}{k\cdot \sqrt{\vol(S_j)\cdot \vol(S_{\ell})}} \\
& = \left( \frac{ \|\beta \|}{\sqrt{k\cdot \max\{\vol(S_j),\vol(S_{\ell}) \}} } - \frac{ \cos\eta\cdot \|\beta \|}{\sqrt{k\cdot \min\{\vol(S_j),\vol(S_{\ell}) \}} }  \right)^2 + \frac{ \|\beta\|^2\cdot \sin^2\eta}{k\cdot \min\{ \vol(S_j), \vol(S_{\ell})\}} \\
& \geq \frac{ \|\beta\|^2\cdot \sin^2\eta}{k\cdot \min\{ \vol(S_j), \vol(S_{\ell})\}} \\
& \geq \frac{ \|\beta\|^2 }{k\cdot \min\{ \vol(S_j), \vol(S_{\ell})\}} \cdot \left( \frac{2\pi\cdot |j-\ell|}{\lceil 2\pi\cdot k\rceil } \cdot\frac{2}{\pi} \right)^2\\
& \geq  \frac{\|\beta\|^2}{3k^3\cdot \min\{\vol(S_j), \vol(S_{\ell})\}},
\end{align*}
where the second inequality holds by the fact that $\sin x\geq (2/\pi)\cdot x$ holds  for any $x\in [0,\pi/2]$. This finishes the proof of the lemma.
\end{proof}

The following lemma will be used to prove Theorem~\ref{thm:apt}. We remark that the following proof  closely follows the similar one from \cite{PSZ17}, however some constants need to be adjusted for our propose. We include the proof here for completeness.

\begin{lemma}\label{lem:permlem}
Let $A_0,\ldots, A_{k-1}$ be a partition of $V$. Assume that, for every permutation $\sigma:\{0,\ldots, k-1\} \rightarrow\{0,\ldots, k-1\}$, there exists some $j'$ such that 
$
\vol\left(A_{j'}\triangle S_{\sigma(j')}\right) \geq \varepsilon \vol\left(S_{\sigma(j')}\right)
$
for some $  48\cdot k^3\cdot (1+\mathsf{APT})\big/ \left(\gamma_k(G)-1 \right)\leq \varepsilon\leq 1/2$, then $
\mathsf{COST}(A_0,\ldots, A_{k-1})\geq 2\mathsf{APT}\big/ \left(\gamma_k(G)-1 \right)$.
\end{lemma}

\begin{proof}
We first consider the case where there exists a permutation $\sigma:\{0,\ldots,k-1\}\rightarrow\{0,\ldots, k-1\}$ such that, for any $0 \le j \le k-1$, 
\begin{equation}
\label{eq:assump}
\vol\left(A_j \cap S_{\sigma(j)}\right) > \frac{1}{2} \vol\left(S_{\sigma(j)}\right).
\end{equation} 
This assumption essentially says that $A_0,\dots,A_{k-1}$ is a non-trivial approximation of the optimal clustering $S_0,\dots,S_{k-1}$ according to some permutation $\sigma$. Later we will show the statement of the lemma trivially holds if no permutations satisfy  \eq{assump}.

 Based on this assumption,  there is $0 \le j' \le k-1$ such that 
$
\vol\left(A_{j'} \triangle S_{\sigma(j')}\right) \ge 2\eps \vol\left(S_{\sigma(j')}\right) $
for some $48 \cdot k^3\cdot \mathsf{APT}\big/ \left(\gamma_k(G)-1 \right)\leq \varepsilon\leq 1/2$. Since
\begin{align*}
\vol\left(A_{j'} \triangle S_{\sigma(j')}\right) &= \vol\left(A_{j'} \setminus S_{\sigma(j')}\right) +\vol\left(S_{\sigma(j')} \setminus A_{j'}\right) \\
& = \sum_{j\neq j'}\vol\left(A_{j'} \cap S_{\sigma(j)}\right) + \sum_{j\neq j'}\vol\left(S_{\sigma(j')} \cap A_{j}\right), 
\end{align*}
 one of the following two cases must hold: 
\begin{enumerate}
\item A large portion of $A_{j'}$ belongs to clusters different from $S_{\sigma(j')}$, i.e., there exist $\eps_0,\dots,\eps_{k-1} \ge 0$ such that $\eps_{j'} = 0$, $\sum_{j=0}^{k-1} \eps_j \ge \eps$, and
$\vol\left(A_{j'} \cap S_{\sigma(j)}\right) \ge\eps_j \vol\left(S_{\sigma(j')}\right)$
for any $0 \le j \le k-1$.
\item $A_{j'}$ is missing a large portion of $S_{\sigma(j')}$, which must have been assigned to other clusters. Therefore, we can define $\eps_0,\dots,\eps_{k-1} \ge 0$ such that $\eps_{j'} = 0$, $\sum_{j=0}^{k-1} \eps_j \ge \eps$, and
$\vol\left(A_j \cap S_{\sigma(j')}\right) \ge \eps_j \vol\left(S_{\sigma(j')}\right)$
for any $0 \le j \le k-1$.
\end{enumerate}

In both cases, we can define sets $B_0,\dots,B_{k-1}$ and $D_0,\dots,D_{k-1}$ such that $B_j$ and $D_j$ belong to the same cluster of the returned clustering but to two different optimal clusters $S_{j_1}$ and $S_{j_2}$.  
More precisely, in the first case, for any $0 \le j \le k-1$, we define $B_j = A_{j'} \cap S_{\sigma(j)}$. We define $D_0,\dots,D_{k-1}$ as an arbitrarily partition of $A_{j'} \cap S_{\sigma({j'})}$ with the constraint that $\vol(D_j) \ge  \eps_j \vol(S_{\sigma({j' })})$. 
This is possible since by \eq{assump} $$\vol\left(A_{j'} \cap S_{\sigma({j'})}\right) \ge \frac{1}{2}  \vol\left(S_{\sigma({j'})}\right) \ge \eps \vol\left(S_{\sigma({j'})}\right).$$ 
In the second case, instead, for any $0\le j \le k-1$, we define $B_j = A_j \cap S_{\sigma({j'})}$ and $D_j = A_j \cap S_{\sigma(j)}$. 
Note that it also holds by \eq{assump} that $\vol(D_j) \ge  \eps_j \vol(S_{\sigma(j)})$. 
We can then combine the two cases together (albeit using different definitions for the sets) and assume that there exist  $\eps_0,\dots,\eps_{k-1} \ge 0$ such that $\eps_{j'} = 0$, $\sum_{j=0}^{k-1} \eps_j \geq \eps$, and such that we can find collections of pairwise disjoint sets $\{B_0,\dots,B_{k-1}\}$ and $\{D_0,\dots,D_{k-1}\}$ with the following properties: for any $j$ there exist indices $\overline{j}$ and $j_1 \ne j_2$ such that 
\begin{enumerate}
\item $B_j,D_j \subseteq A_{\overline{j}}$
\item $D_j \subseteq S_{j_1}, B_j \subseteq S_{j_2}$
\item $\vol(B_j) \ge \eps_j \min\{\vol\left(S_{j_1}\right), \vol\left(S_{j_2}\right)\}$
\item $\vol(D_j) \ge \eps_j \min\{\vol\left(S_{j_1}\right), \vol\left(S_{j_2}\right)\}$
\end{enumerate}

For any $j$, we define $c_j$ as the centre of the corresponding cluster $A_{\overline{j}}$ to which both $B_j$ and $D_j$ are subset of. We can also assume without loss of generality that  $\left\|c_j-p^{(j_1)}\right\| \ge \left\|c_j-p^{(j_2)}\right\|$ which implies $$\left\|p^{({j_1})} - c_j\right\| \ge\frac{1}{2}\cdot \left\|p^{(j_1)} - p^{(j_2)}\right\|.$$ As a consequence, points in  $B_j$ are far away from $c_j$. Notice that if instead $\left\|c_j-p^{(j_1)}\right\| <\left\|c_j-p^{(j_2)}\right\|$, we would just need to reverse the role of $B_j$ and $D_j$ without changing the proof. We now bound $\mathsf{COST}(A_0,\dots,A_{k-1})$ by looking only at the contribution of the points in the $B_j$'s. 
Therefore, we have that
\begin{align*}
\mathsf{COST}(A_0,\dots,A_{k-1}) = \sum_{j=0}^{k-1} \sum_{u\in A_j } d_u  \|F(u) - c_j\|^2 \ge \sum_{j=0}^{k-1} \sum_{u\in B_j } d_u  \|F(u) - c_j\|^2.  
\end{align*}
By applying the inequality $a^2 + b^2 \ge (a-b)^2/2$, we have that
 \begin{align}
\mathsf{COST}(A_0,\dots,A_{k-1}) 	&\ge \sum_{j=0}^{k-1} \sum_{u\in B_j } d_u \left( \frac{\left\|p^{({j_1})} - c_j\right\|^2}{2} - \left\|F(u) - p^{(j_1)}\right\|^2\right)  \notag\\
	&\ge \sum_{j=0}^{k-1} \sum_{u\in B_j } d_u  \frac{\left\|p^{({j_1})} - c_j\right\|^2}{2} - \sum_{j=0}^{k-1}\sum_{u\in B_j } d_u \left\|F(u) - p^{(j_1)}\right\|^2 \notag \\
	&\ge \sum_{j=0}^{k-1} \sum_{u\in B_j } d_u  \frac{\left\|p^{({j_1})} - c_j\right\|^2}{2} - \frac{1}{\gamma_k(G)-1} \notag \\
	&\ge\sum_{j=0}^{k-1} \sum_{u\in B_j } d_u  \frac{\left\|p^{({j_1})} - p^{({j_2})}\right\|^2}{8} - \frac{1}{\gamma_k(G)-1} \notag \\
	&\ge \sum_{j=0}^{k-1} \frac{ \|\beta\|^2\cdot \vol(B_j)}{ 24k^3\cdot \min{\{\vol(S_{j_1}),\vol(S_{j_2})\}}} - \frac{1}{\gamma_k(G)-1} \notag \\
	&\ge \sum_{j=0}^{k-1}\frac{\|\beta\|^2\cdot \eps_j \min{\{\vol(S_{j_1}),\vol(S_{j_2})\}}}{ 24 k^3\cdot \min{\{\vol(S_{j_1}),\vol(S_{j_2})\}}} - \frac{1}{\gamma_k(G)-1} \notag \\
	&\ge \sum_{j=0}^{k-1} \frac{ \eps_j\cdot \|\beta\|^2 }{ 24k^3} - \frac{1}{\gamma_k(G)-1} \notag \\
	&\ge \frac{ \eps }{ 24k^3} - \frac{1}{\gamma_k(G)-1}  \notag\\
	& \geq \frac{1}{24k^3}\cdot \frac{48k^3\cdot (1+\mathsf{APT})}{\gamma_k(G)-1}  - \frac{1}{\gamma_k(G)-1}\notag\\
	&\ge  \frac{2\APT}{\gamma_k(G)-1}. \notag
\end{align}

 It remains to show that removing assumption \eq{assump} implies the Lemma as well. Notice that if \eq{assump} is not satisfied, for all permutations $\sigma$ there exists $0 \le \ell^{\star} \le k-1$ such that $\vol\left(A_{ \ell^{\star}} \cap S_{ \sigma(\ell^{\star})} \right) \le \frac{1}{2} \vol\left(S_{\sigma( \ell^{\star})}\right)$. We can also assume the following stronger condition:
 \begin{equation}
 \label{eq:assump2}
 \vol\left(A_{ \ell^{\star}} \cap S_{j} \right) \le \frac{1}{2} \vol\left(S_{ j}\right) \qquad \text{ for any } 0 \le j \le k-1.
 \end{equation}
 Indeed, if there would exist a unique $j \neq \sigma( \ell^{\star})$ such that $\vol\left(A_{ \ell^{\star}} \cap S_{j} \right) > \frac{1}{2} \vol\left(S_{ j}\right)$, then it would just mean that $\sigma$ is the ``wrong'' permutation and we should consider only permutations $\sigma' \neq \sigma$ such that $\sigma'(\ell^{\star}) = j$. If instead there would exist $j_1 \neq j_2$ such that  $\vol\left(A_{ \ell^{\star}} \cap S_{j_1} \right) > \frac{1}{2} \vol\left(S_{j_1}\right)$ and $\vol\left(A_{ \ell^{\star}} \cap S_{j_2} \right) > \frac{1}{2} \vol\left(S_{ j_2}\right)$, then it is easy to see that the Lemma would hold, since in this case $A_{ \ell^{\star}}$ would contain large portions of two different optimal clusters, and, as clear from the previous part of the proof, this would imply a high $k$-means cost.

Therefore, we just need to show that the statement of the Lemma holds when \eq{assump2} is satisfied. For this purpose we  define sets $C_0, \dots, C_{k-1}$ which are subsets of vertices in $S_0, \dots, S_{k-1}$ that are close in the spectral embedding to $p^{(0)},\dots,p^{(k-1)}$. Formally, for any $0 \le j \le k-1$,
\[
C_j = \left\{ u \in S_j \, \colon \, \|F(u) - p^{(j)}\|^2 \le \frac{100}{\vol(S_j)}\cdot \left(\gamma_k(G) -1\right)   \right\}.
\]
Notice that by \lemref{distance_p_x} $\vol(C_j) \ge \frac{99}{100} \vol(S_j)$. By assumption \eq{assump2}, roughly half of the volume of all the $C_j$'s must be contained in at most $k-1$ sets (all the $A_j$'s different from $A_{\ell^{\star}}$). We prove this implies that the $k$-means cost is high, from which the Lemma follows.

Let $c_0,\dots,c_{k-1}$ be the centres of $A_0,\dots,A_{k-1}$. We are trying to assign a large portion of each of the $k$ optimal clusters to only $k-1$ centres (namely all the centres different from $c_{\ell^{\star}}$). Moreover,   any centre $c_j \neq c_{\ell^{\star}}$ can either be close to $p^{(\ell^{\star})}$ or to another optimal centre $p^{(j')}$, but not to both. As a result, there will be at least one $C_j$ whose points are assigned to a centre which is at least $\Omega(1/\vol(S_j))$ far from $p^{(j)}$ (in squared Euclidean distance). Therefore, by the definition of $C_j$ and the fact that $\vol(C_j) \ge \frac{99}{100} \vol(S_j)$, the $k$-means cost is at least $\Omega\left(\frac{1}{\vol(S_j)} \cdot \vol(C_j)\right) = \Omega(1).$
This concludes the proof.
\end{proof}

\begin{proof}[Proof of Theorem~\ref{thm:apt}]
 Assume for contradiction that, for any permutation $\sigma:\{0,\ldots, k-1\} \rightarrow \{0,\dots, k-1\}$, there is an index $j\in\{0,\ldots, k-1\}$ such that   $\vol\left(A_j \triangle S_{\sigma(j)}\right) \geq \varepsilon \vol\left(S_{\sigma(j)}\right)$. Then, by Lemma~\ref{lem:permlem} we have that $\mathsf{COST}(A_0,\ldots, A_{k-1})\geq 2\mathsf{APT}\big/ \left( \gamma_k(G)-1 \right)$, which contradicts the fact that $\mathsf{COST}(A_0,\ldots, A_{k-1})\leq \mathsf{APT}\big/ \left( \gamma_k(G)-1 \right)$.
\end{proof}
 
Now we prove Theorem~\ref{thm:distributed_sparsification}. The following two technical lemmas will be used in our proof.

 \begin{lemma}[Bernstein's Inequality, \cite{chung2006concentration}]
\label{lem:bernstein}
 Let $X_1,...X_n$ be independent random variables such that $\abs{X_i} \leq M$ for any $i \in \{1,...,n\}$. Let $X=\sum_{i=1}^{n}X_i$ and let $R = \sum_{i=1}^{n}\mathbb{E}[X_i^2]$. Then, it holds that
\begin{equation*}
    \mathbb{P}\left[\abs{X - \mathbb{E}[X]} \geq t \right] \leq 2\cdot \mathrm{exp}\left(-\frac{t^2}{2(R + Mt/3)}\right).
\end{equation*}
\end{lemma}

\begin{lemma}[Matrix Chernoff Bound, \cite{tropp}]
\label{lem:chernoff}
Consider a finite sequence $\{X_i\}$ of independent, random, PSD matrices of dimension $d$ that satisfy $\|X_i\| \le R$. Let $\mu_{\min} \triangleq \lambda_{\min}\left(\Ex{\sum_i X_i}\right)$ and $\mu_{\max} \triangleq \lambda_{\max}\left(\Ex{\sum_i X_i}\right)$. Then it holds that 
\begin{align*}
\Pro{\lambda_{\min}\left({\sum_i X_i}\right) \le (1-\delta)\mu_{\min}} &\le d \cdot \left(\frac{\mathrm{e}^{-\delta}}{(1-\delta)^{1-\delta}}\right)^{\mu_{\min}/R} \text{  for } \delta \in [0,1], \text{ and} \\
\Pro{\lambda_{\max}\left({\sum_i X_i}\right) \ge (1 +\delta)\mu_{\max}} &\le d \cdot \left(\frac{\mathrm{e}^{\delta}}{(1+\delta)^{1+\delta}}\right)^{\mu_{\max}/R} \text{  for } \delta \ge 0.
\end{align*}
\end{lemma}
 
 \begin{proof}[Proof of Theorem~\ref{thm:distributed_sparsification}]   
We first analyse the size of $F$.
Since 
\[
\sum_{u\in V} \sum_{e=(u,v) }  w(u,v)\cdot \frac{\alpha\log n}{d^{\text{out}}_u\cdot\lambda_{2}}  = O\left(\frac{n\log n}{\lambda_{2}}\right),
\]
and
\[
\sum_{v\in V} \sum_{e=(u,v) }  w(u,v)\cdot \frac{\alpha\log n}{d^{\text{in}}_v\cdot\lambda_{2}}  = O\left(\frac{n\log n}{\lambda_{2}}\right),
\]
it holds by Markov inequality that the number of edges $e=(u,v)$    with $ w(u,v)\cdot \frac{\alpha\log n}{d^{\text{out}}_u\cdot\lambda_{2}} \geq 1$ and $w(u,v)\cdot \frac{\alpha\log n}{d^{\text{in}}_v\cdot\lambda_{2}}\geq 1$ is $O\left( \frac{n\log n}{\lambda_{2}}\right)$. Without loss of generality, we  assume that these edges are in $F$, and  in the remaining part of the proof we assume it holds for any edge $e = (u,v)$ that
\[
w(u,v) \cdot \frac{\alpha\cdot \log{n}}{ d^{\text{out}}_u\cdot\lambda_{2}} < 1, \qquad 
w(u,v) \cdot \frac{\alpha\cdot \log{n}}{ d^{\text{in}}_v\cdot\lambda_{2}} < 1.\]
Moreover, the expected number of edges in $H$ equals to
\begin{align*}
\sum_{e = (u,v )} p_e \leq \sum_{e = (u,v)} p_u(u,v) + p_v(u,v)  &=  \frac{\alpha \cdot \log{n}}{\lambda_2}\sum_{e = (u,v)} \left( \frac{w(u,v)}{ d^{\text{out}}_u} +\frac{w(u,v)}{ d^\text{in}_v}   \right) \\
&= O\left(\frac{n\log n}{\lambda_{2}}\right),
\end{align*}
and thus by Markov's inequality we have that with constant probability the number of sampled edges $|F|=  O\left( \left(1/\lambda_{2}\right)\cdot n\log n\right)$.  

\textcolor{black}{\emph{Proof of $\theta_k(H)=\Omega(\theta_k(G))$.}}  Next we show that the sparsified graph constructed by the algorithm preserves $\theta_k(G)$ up to a constant factor. Without loss of generality, let $S_0,\ldots, S_{k-1}$ be the optimal $k$ clusters such that 
\[
\Phi_G(S_0,\ldots, S_{k-1})  = \theta_k(G).
\]
For any edge $e=(u,v)$ satisfying $u\in S_{j}$ and $v\in S_{j-1}$ for some $1\leq j\leq k-1$, we define a random variable $Y_e$ by
\[
    Y_e=\begin{cases}
      w(u,v)/p_e & \text{with probability $p_e$,}\\
      0 & \text{otherwise.}
    \end{cases}
\]
We also define  random variables $Z_1,\ldots, Z_{k-1}$, where $Z_j~(1\leq j\leq k-1)$ is defined  by
\[
Z_j =  \sum_{\substack{ e=\{u,v\}\in E[G] \\ u\in S_j, v\in S_{j-1} }} Y_e.
\]
By definition, we have that 
\[
\mathbb{E}[Z_j] =   \sum_{\substack{ e=\{u,v\}\in E[G] \\ u\in S_j, v\in S_{j-1} }} \mathbb{E}[Y_e] =   \sum_{\substack{ e=\{u,v\}\in E[G] \\ u\in S_j, v\in S_{j-1} }} w(u,v)= w(S_j, S_{j-1}).
\]
Moreover, we look at  the second moment and have that
\begin{align*}
      \sum_{\substack{ e=\{u,v\}\in E[G] \\ u\in S_j, v\in S_{j-1} }} \mathbb{E}\left[Y_e^2\right] & =  \sum_{\substack{ e=\{u,v\}\in E[G] \\ u\in S_j, v\in S_{j-1} }}  p_e\cdot \left( \frac{w(u,v)}{p_e}\right)^2  \\
     & =   \sum_{\substack{ e=\{u,v\}\in E[G] \\ u\in S_j, v\in S_{j-1} }}     \frac{(w(u,v))^2}{p_e}\\
     & \leq   \sum_{\substack{ e=\{u,v\}\in E[G] \\ u\in S_j, v\in S_{j-1} }}     \frac{(w(u,v))^2}{w(u,v)}\cdot \frac{\lambda_2\cdot d_u^{\mathrm{out}}}{\alpha\log n} \\
     & = \frac{\lambda_2}{\alpha\log n}\cdot \sum_{\substack{ e=\{u,v\}\in E[G] \\ u\in S_j, v\in S_{j-1} }} w(u,v)\cdot d_u^{\mathrm{out}}\\
     & \leq \frac{\lambda_2}{\alpha\log n}\cdot\Delta_j^{\mathrm{out}}\cdot w(S_j, S_{j-1}),
\end{align*}
where  $\Delta_j^{\mathrm{out}}$ is the maximum of the out degree of vertices in $S_j$ and the first inequality follows by  the fact that
\[
p_e = p_u(u,v) + p_v(u,v) - p_u(u,v)p_v(u,v) \geq p_u(u,v) = w(u,v)\cdot \frac{\alpha\log n}{\lambda_2\cdot d_u^{\mathrm{out}}}.
\]
In addition, it holds for any $e=(u,v), u\in S_j, v\in S_{j-1}$ that
\[
\left|\frac{w(u,v)}{p_e} \right| \leq
\left|\frac{w(u,v)}{p_u(u,v)} \right| \leq\frac{\lambda_2\cdot \Delta_j^{\mathrm{out}}}{\alpha\cdot\log n}.
\]
We apply Bernstein's Inequality~(Lemma~\ref{lem:bernstein}), and obtain for any $1\leq j\leq k-1$ that
\begin{align*}
    &\mathbb{P}\left[ |Z_j - w(S_j, S_{j-1})| \geq  (1/2)\cdot w(S_j, S_{j-1}) \right]\\
    & = \mathbb{P}\left[ |Z_j - \mathbb{E}[Z_j]| \geq  (1/2)\cdot  \mathbb{E}[Z_j]  \right]\\
    & \leq 2\cdot\mathrm{exp}\left( - \frac{\mathbb{E}[Z_j]^2/4}{2 \left(\frac{\lambda_2}{\alpha\log n}\cdot\Delta_j^{\mathrm{out}}\cdot w(S_j, S_{j-1}) + \frac{\lambda_2\cdot \Delta_j^{\mathrm{out}}}{\alpha\cdot\log n}\cdot \frac{1}{6}\cdot w(S_j, S_{j-1})  \right) } \right)\\
    & \leq 2\cdot\mathrm{exp}\left( - \frac{\alpha\cdot\log n\cdot \mathbb{E}[Z_j]}{10\cdot\lambda_2\cdot\Delta_j^{\mathrm{out}}   } \right). \\
\end{align*}
Hence, with high probability cut values $w(S_j, S_{j-1})$ for all $1\leq j\leq k-1$ are approximated up to a constant factor. Using   the same   technique, we can show that with high probability the volumes of all the sets $S_0,\ldots, S_{k-1}$ are approximately preserved in $H$ as well. Combining this with the definition of $\Phi$, we have that $\Phi_G(S_0,\ldots, S_{k-1})$ and $\Phi_H(S_0,\ldots, S_{k-1})$ are approximately the same up to a constant factor. Since $S_0, \ldots, S_{k-1}$ are the sets that maximising the value of $\theta_k(G)$, we have that   $\theta_k(H) = \Omega(\theta_k(G))$.

\emph{Proof of $\lambda_2\left(\mathcal{L}_H\right)=\Omega(\lambda_2\left(\mathcal{L}_G\right))$.}  Finally,  we prove that the top $n-1$ eigenspace is approximately preserved in $H$. Let $\overline{\calL}_G$ be the projection of $\calL_G$ on its top $n-1$ eigenspaces. We can write $\overline{\calL}_G$ as   \[
\overline{\calL}_G = \sum_{i=2}^n \lambda_i f_i f_i^{*}.
\]
 With a slight abuse of notation we call $\overline{\calL}_{G}^{-1/2}$ the square root of the pseudoinverse of $\overline{\calL}_{G}$, i.e., 
 \[
 \overline{\calL}_{G}^{-1/2} =  \sum_{i=2}^n (\lambda_i)^{-1/2} f_i f_i^{*}.
 \]
We call $\overline{\mathcal{I}}$ the projection on $\Span\{f_{2},\dots,f_n\}$, i.e.,
 \[
 \overline{\mathcal{I}}= \sum_{i=2}^n f_i f_i^{*}.
 \]
We will  prove that the top $n-1$ eigenspaces of $\calL_G$ are preserved.
To prove this,  recall that the probability of any edge $e=(u,v)$ being sampled in $H$ is 
\[
p_e=p_u(u,v) + p_v(u,v)- p_u(u,v)\cdot p_v(u,v),
\]
and it holds that 
 $\frac{1}{2}(p_u(u,v)+p_v(u,v)) \le p_e \le p_u(u,v) + p_v(u,v)$. Now for each edge $e=( u,v)$ of $G$ we define a random matrix $X_e\in\mathbb{C}^{n\times n}$ by 
 \[
X_e = 
\begin{cases}
  w_H(u,v)\cdot  \overline{\calL}_{G}^{-1/2} b_{e} b_{e}^{*} \overline{\calL}_{G}^{-1/2} & \text{if\ } e=(u, v) \text{\ is sampled by the algorithm}, \\
  0          & \text{otherwise,}\ 
\end{cases}
\]
where the vector $b_e$ is defined by 
 $b_e = \left(
\omega_{2\lceil 2\pi\cdot k\rceil}\chi_u - \omega^*_{2\lceil 2\pi\cdot k\rceil}\chi_v\right)$ and for any vertex $u$ the normalised indicator vector $\chi_u$ is defined by $\chi_u(u)=1/\sqrt{d_u}$, and $\chi_u(v) =0$ for any $v\neq u$.
 Notice that
\[
\sum_{e\in E[G]} X_e = \sum_{ \mathrm{sampled\ edges\ } e=(u,v)} w_H(u,v)\cdot  \overline{\calL}_{G}^{-1/2} b_{e} b_{e}^{*} \overline{\calL}_{G}^{-1/2}= \overline{\calL}_{G}^{-1/2} \calL_H' \overline{\calL}_{G}^{-1/2},
\]
where  it follows by definition that\[
\calL_H' = \sum_{ \mathrm{sampled\ edges\ } e=(u,v)} w_H(u,v)\cdot  b_{e} b_{e}^{*}\]
  is essentially the  Laplacian matrix of $H$ but is normalised with respect to the degrees of the vertices in the  original graph $G$, i.e., 
   $\mathcal{L}_H' = D_G^{-1}D_H - D_G^{-1/2}A_H D_G^{-1/2}$. We will prove that, with high probability, the top $n-1$ eigenspaces of  $\calL_H'$ and $\calL_G$ are approximately the same. Later we will show the same holds for $\calL_H$ and $\calL_H'$, which implies that $\lambda_{2}(\calL_H') = \Omega( \lambda_{2} (\calL_G))$.
  
   We will use the matrix Chernoff bound for our proof. We start looking at the first moment of the expression above:
\begin{align*}
\Ex{\sum_{e \in E} X_e} & = \sum_{e=(u,v) \in E[G]} p_e\cdot w_H(u,v)
\cdot  \overline{\calL}_{G}^{-1/2} b_{e} b_{e}^{*} \overline{\calL}_{G}^{-1/2} \\
& =
\sum_{e=(u,v) \in E[G]} p_e\cdot \frac{w(u,v)}{p_e}
\cdot  \overline{\calL}_{G}^{-1/2} b_{e} b_{e}^{*} \overline{\calL}_{G}^{-1/2}\\
&
= \overline{\calL}_G^{-1/2} \calL_G \overline{\calL}_G^{-1/2} = \overline{\mathcal{I}}.
\end{align*}
Moreover, for any sampled $e =(u,v) \in E$ we have that 
\begin{align*}
\|X_e\| 	&\leq   w_H(u,v)\cdot b_{e}^{*} \overline{\calL}_{G}^{-1/2} \overline{\calL}_{G}^{-1/2} b_{e} =  
		\frac{w(u,v)}{p_{e}}\cdot b_{e}^{*} \overline{\calL}_{G}^{-1}  b_{e} \le \frac{w(u,v)}{p_e}\cdot \frac{1}{\lambda_{2}}\cdot  \|b_e\|^2  \\
		&\le \frac{2\lambda_{2}}{\alpha \cdot  \log{n}\cdot\left(\frac{1}{d^{\text{out}}_u}+\frac{1}{d^{\text{in}}_v}\right)} \cdot \frac{1}{\lambda_{2}}  \left(\frac{1}{d^{\text{out}}_u}+\frac{1}{d^{\text{in}}_v}\right) \le \frac{2}{\alpha \log{n}},
\end{align*}
where the second inequality follows by the min-max theorem of  eigenvalues.   Now we apply the matrix Chernoff bound~(Lemma~\ref{lem:chernoff}) to analyse the eigenvalues of $\sum_{e\in E} X_e$, and build a connection between $\lambda_{2}(\mathcal{L}'_H)$ and $\lambda_{2}(\mathcal{L}_G)$.  By setting the parameters of Lemma~\ref{lem:chernoff} by 
$\mu_{\max}=\lambda_{\max}\left(\Ex{\sum_{e\in E[G]} X_e } \right) = \lambda_{\max}\left( \overline{\mathcal{I}}  \right) = 1$, $R= 2/\left( \alpha\cdot \log n \right)$ and $\delta=1/2$, we have that
\[
\mathbb{P}\left[\lambda_{\max}\left(\sum_{e\in E[G]} X_e \right) \ge 3/2 \right] \leq n\cdot \left( \frac{\mathrm{e}^{1/2}}{ \left( 1+1/2 \right)^{3/2}} \right)^{\alpha\log n/2} = O\left(1/n^ c \right)
\] 
for some constant $c$. This gives us that 
\begin{equation}\label{eq:boundlambdamax}
\mathbb{P}\left[\lambda_{\max}\left(\sum_{e\in E[G]} X_e \right) \leq 3/2 \right] =1-O(1/n^c).
\end{equation}
On the other side, since our goal  is to analyse $\lambda_{2}(\mathcal{L}'_H)$ with respect to $\lambda_{2}(\mathcal{L}_G)$, it suffices to work with the top $(n-1)$ eigenspace of $\mathcal{L}_G$.  Since $\Ex{\sum_{e \in E} X_e}=\overline{\mathcal{I}}$, we can assume without loss of generality that $\mu_{\min}=1$. Hence, by setting $R= 2/\left(\alpha \cdot \log n \right)$ and $\delta=1/2$, we have that
\[
\mathbb{P}\left[\lambda_{\min}\left(\sum_{e\in E[G]} X_e \right) \leq 1/2 \right] = n\cdot\left( \frac{\mathrm{e}^{-1/2}}{ (1/2)^{1/2}} \right)^{\alpha \log n/2} = O\left(1/n^c\right)
\]
for some constant $c$. This gives us that 
\begin{equation}\label{eq:boundlambdamin}
\mathbb{P}\left[\lambda_{\min}\left(\sum_{e\in E[G]} X_e \right) > 1/2 \right] =1-O(1/n^c).
\end{equation}
Combining \eq{boundlambdamax}, \eq{boundlambdamin}, and the fact of $\sum_{e\in E[G]} X_e =\overline{\calL}_{G}^{-1/2} \calL_H' \overline{\calL}_{G}^{-1/2} $, with probability  $1-O\left(1/n^c\right)$ it holds for any  non-zero $x\in\mathbb{C}^n$ in the space spanned by $f_{2},\ldots, f_n$ that
\begin{equation}\label{eq:twoside}
\frac{x^{*} \overline{\calL}_{G}^{-1/2} \calL_H' \overline{\calL}_{G}^{-1/2} x}{x^{*} x}\in \left(1/2, 3/2\right).
\end{equation}
By setting $y=\overline{\calL}_G^{-1/2}x$, we can rewrite \eq{twoside} as
\[
\frac{y^{*}\calL_H' y }{y^{*} \overline{\calL}_G^{1/2}\overline{\calL}_G^{1/2} y} = \frac{y^{*}\calL_H' y}{y^{*} \overline{\calL}_G  y} =\frac{y^{*}\calL_H' y}{y^{*}y} \frac{y^{*} y}{y^{*} \overline{\calL}_G  y} \in (1/2, 3/2).
\] 
Since $\Dim(\Span\{f_{2},\dots,f_n\}) = n-1$, we have just proved there exist $n-1$ orthogonal vectors whose Rayleigh quotient with respect to $\calL'_H$ is $\Omega(\lambda_{2}(\mathcal{L}_G))$. By the Courant-Fischer Theorem,  we have 
\begin{equation}\label{eq:b1}
\lambda_{2}(\mathcal{L}_H') \ge \frac{1}{2} \lambda_{2}(\mathcal{L}_G).
\end{equation}

It remains to show that
$\lambda_{2}(\mathcal{L}_H) = \Omega\left(\lambda_{2}(\mathcal{L}'_H) \right)$, which implies that  $\lambda_{2}(\mathcal{L}_H) = \Omega\left(\lambda_{2}(\mathcal{L}_G) \right)$ by \eq{b1}.
By the definition of $\mathcal{L}'_H$, we have that for the Laplacian
 $\mathcal{L}_H = D_H^{-1/2} D_G^{1/2}  \mathcal{L}_H' D_G^{1/2} D_H^{-1/2}$. Therefore, for any $x \in \mathbb{C}^n$ and $y =  D_G^{1/2}  D_H^{-1/2} x$, it holds that 
\begin{equation}\label{eq:extra}
\frac{x^{*} \mathcal{L}_H x}{x^{*} x} = \frac{y^{*} \mathcal{L}_H' y}{x^{*} x}\geq \frac{1}{2} \cdot \frac{y^{*} \mathcal{L}_H' y}{y^{*} y},
\end{equation}
where the last equality follows from the fact that the degrees in $H$ and $G$ differ just by a constant multiplicative factor, and therefore,
\[
y^{*} y = \left(D_G^{1/2}  D_H^{-1/2} x\right)^{*} \left(D_G^{1/2}  D_H^{-1/2} x\right) = x^{*} D_G  D_H^{-1} x \geq \frac{1}{2} \cdot  x^{*} x.
\]
Finally, we show that \eq{extra} implies that $\lambda_{2}(\calL_H)\geq (1/2)\cdot\lambda_{2} (\calL'_H)$. To see this,  let $S_1 \subseteq \mathbb{C}^n$ be a $(2)$-dimensional subspace of  $\mathbb{C}^n$ such that
\[
\lambda_{2}(\calL_{H}) = \max_{x \in S_1} \frac{x^{*} \calL_{H} x}{x^{*} x}.
\]
Let $S_2 = \left\{ D^{1/2}_G D^{-1/2}_H x \colon x \in S_1\right\}$. Notice that  since $D^{1/2}_G D^{-1/2}$ is   full rank, $S_2$ has dimension $2$. Therefore,
\begin{equation}\label{eq:extra11}
\lambda_{2}(\calL'_H) = \min_{S\colon \dim(S) = 2} \max_{y \in S} \frac{y^{*} \calL'_H y}{y^{*} y} \le \max_{y \in S_2} \frac{y^{*} \calL'_H y}{y^{*} y} \le 2 \max_{x \in S_1} \frac{x^{*} \calL_H x}{x^{*} x} = 2 \lambda_{2}(\calL_H),
\end{equation}
where the last inequality follows by \eq{extra}. Combining \eq{b1} with \eq{extra11} gives us that   $\lambda_{2}(\calL_H)=\Omega(\lambda_{2}(G))$. This concludes the proof.   
\end{proof}
 
\section{Omitted details from Section~\ref{sec:experiments}}\label{sec:appendix_experiments}

\subsection{UN Comtrade Data Preparation}\label{sec:app_unc_dataprep}
The API provided by the UN gives a lot of flexibility on the type of selected data. 
It is possible to specify the \emph{product type} to either trade in goods (e.g., oil, wood, and appliances) or services (e.g.,  financial services, and construction services). 
Moreover, the \emph{classification code} can be selected, which we set to the Harmonised System (HS).
The HS categorises goods according to a $6$-digit  classification code (e.g.,  $060240$, where the first two digits ``$06$'' represents ``plants'', the second two  digits ``$02$'' represents ``alive'', and the last two digits ``$40$'' code for ``roses''). 
The \emph{reporting} countries and \emph{partner} countries can also be specified, where the reporting country reports about its own reported tradeflow with partner countries. 
The settings we used to download the data for our experiments were Goods on an annual frequency, the HS code as reported, over the period from 2002 to 2017,  with all reporting and all partner countries, all trade flows and all HS commodity codes. The total size of the data in zipped files is $99.8$GB, where each csv file (for every year) contains around $20,000,000$ lines.

For every  pair of countries $j$ and $\ell$, where $j$ is the reporting country and $\ell$ is the partner country, the database contains the amount that country $j$ imports from country $\ell$ for a specific commodity, and also the amount $j$ exports to $\ell$. 
There are several cases where countries $j$ and $\ell$ report different trading amounts with each other. 
Usually, the larger value is considered more accurate and is   used instead of the average~\cite{DITTRICH20101838}. 
To construct the digraph of the world trade network and its corresponding adjacency matrix, we fill in each entry of the adjacency matrix $M^c$ for commodity $c$ as follows: for each pair of countries $j$ and $\ell$, we compute $d^c_{j\ell} = e^c_{j\ell} - e^c_{\ell j}$, where $e^c_{j\ell}$ is the amount country $j$ exports to country $\ell$ for commodity $c$. 
If $d^c_{j\ell} > 0$, we set $M^c_{j\ell} = d^c_{j\ell}$ and $M^c_{\ell  j} = 0$. 
If $d^c_{j\ell} < 0$ (and thus $d^c_{\ell j} > 0$), we set $M^c_{\ell j} = d^c_{\ell j}$ and $M^c_{j \ell} = 0$. 

For our experiments we investigate the trade in ``Mineral Fuels, mineral oils, and products of their distillation'' (HS code 27), and the trade in ``Wood and articles of wood'' (HS code 44).

\subsection{DD-SYM Plots International Oil Trade}\label{sec:appendix_ddsym_clustervisualisations}
\begin{figure}[h]
\centering
    \begin{minipage}{0.5\textwidth}
      \centering
    \includegraphics[width=\textwidth]{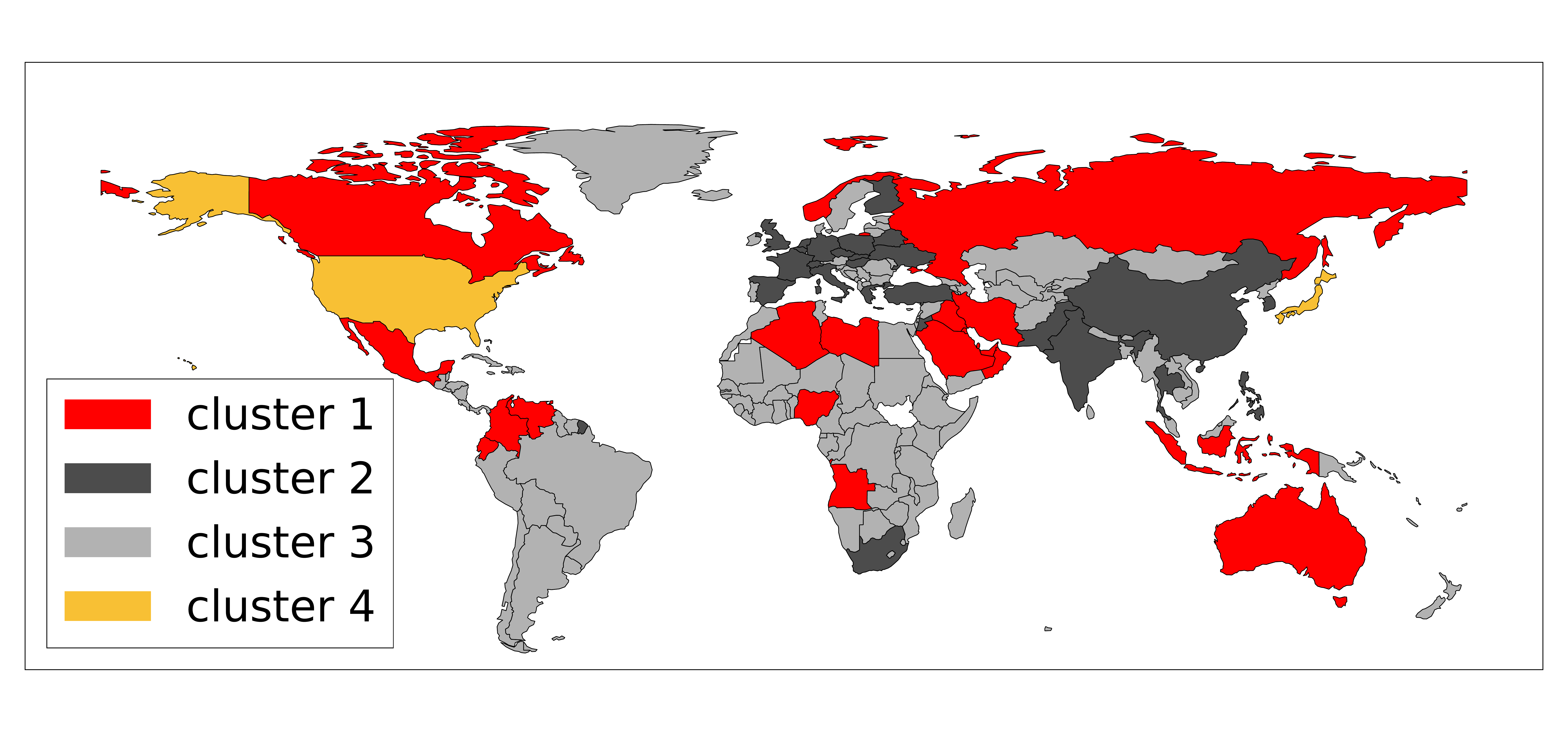}
    \vspace{-0.8cm}
    \caption*{2006}
    \end{minipage}%
    \begin{minipage}{0.5\textwidth}
      \centering  
    \includegraphics[width=\textwidth]{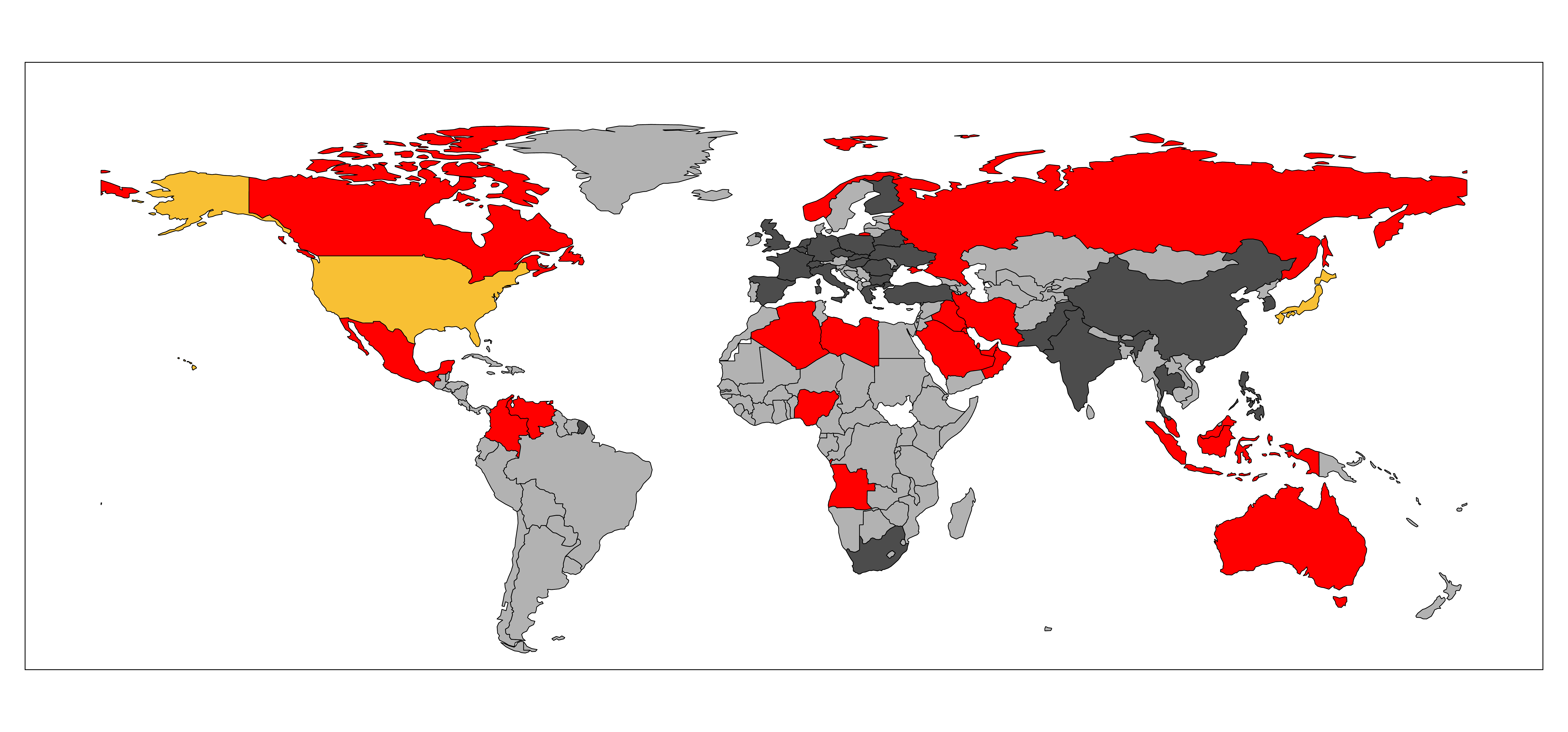}
    \vspace{-0.8cm}
    \caption*{2007}
    \end{minipage}
    \begin{minipage}{0.5\textwidth}
      \centering
    \includegraphics[width=\textwidth]{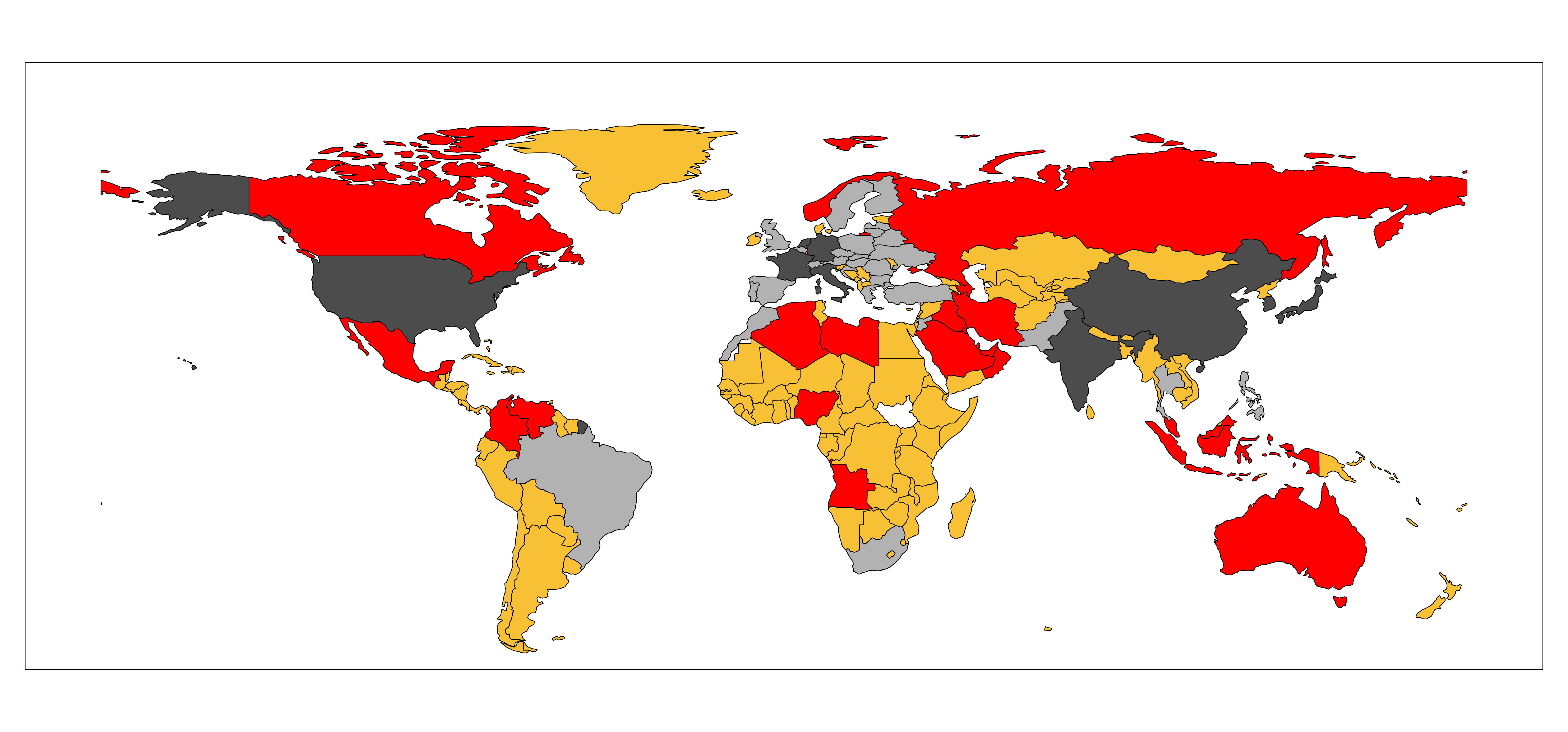}
    \vspace{-0.8cm}
    \caption*{2008}
    \end{minipage}%
    \begin{minipage}{0.5\textwidth}
      \centering
    \includegraphics[width=\textwidth]{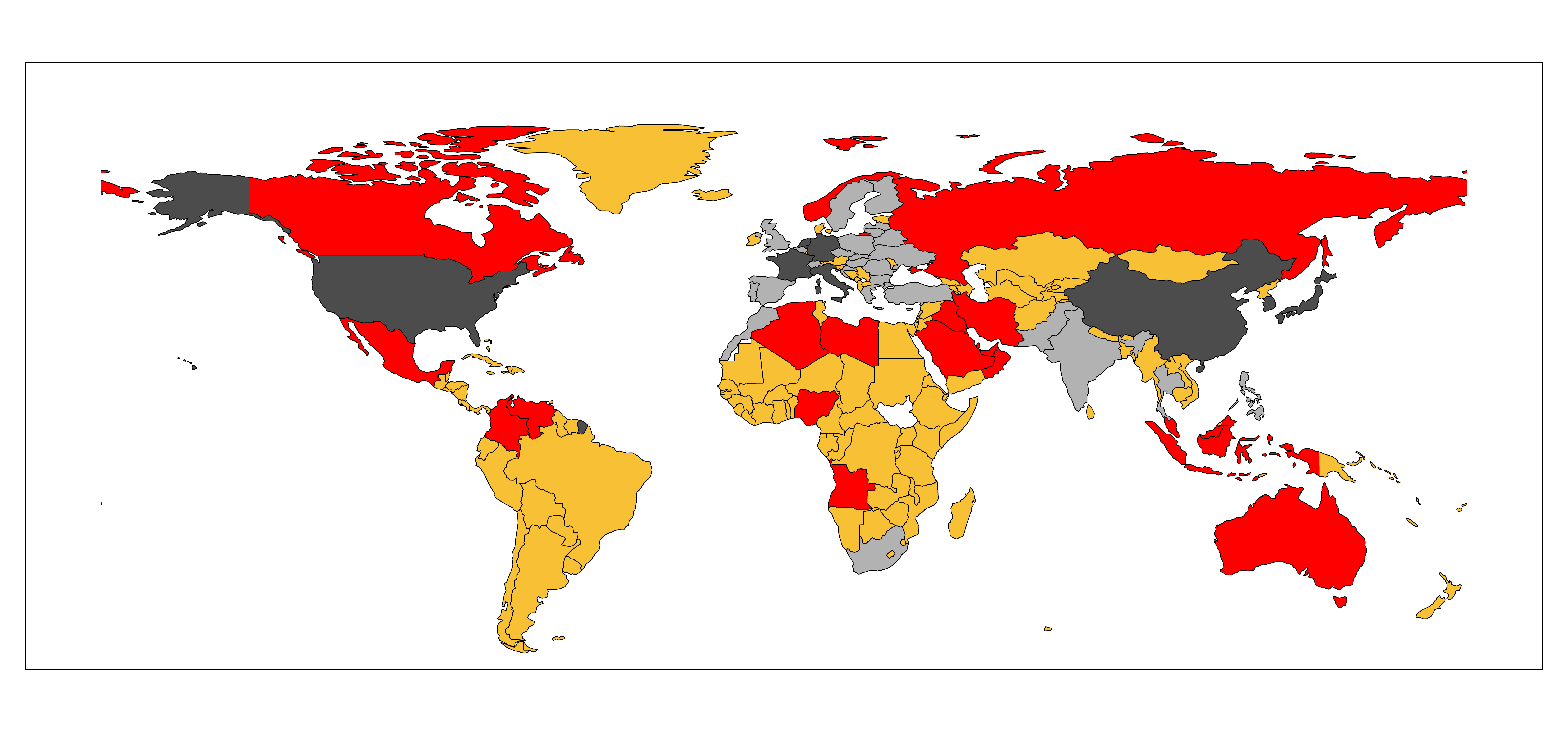}
    \vspace{-0.8cm}
    \caption*{2009}
    \end{minipage}%
\caption{Change in clustering of the IOTN over period 2006--2009 with $k=4$ using \texttt{DD-SYM} method.   Red countries form the start of the trade chain, and yellow countries the end of the trade chain. Countries   coloured white have no data.} \label{fig:oil_cluset_change_20062009_dd_sym}
\end{figure} 

We   plot the cluster visualisations for the \texttt{DD-SYM} algorithm in Figure~\ref{fig:oil_cluset_change_20062009_dd_sym} on the international oil trade network, over the period 2006-2009. The clusters between 2006 and 2007 are almost identical, and then there is a shift in  the clustering structure between 2007 and 2008. This change occurs one year before the change in the \texttt{SimpleHerm} method, and this change is also one year earlier than the changes found in the complex network analysis literature~\cite{AN2014254, ZHONG201442}. This indicates that the \texttt{SimpleHerm} clustering result is more in line with other literature.

\subsection{International Wood Trade}


For comparison we visualise the clustering result of the \texttt{DD-SYM} method over the period of 2006 -- 2009, see Figure~\ref{fig:wood_cluset_change_20062009_ddsym}. In addition, Figure~\ref{fig:wood_symmetric_difference} compares the symmetric difference of the clusters returned by different algorithms over the consecutive years. 
  Again, we notice  that our algorithm finds a peak around the economic crisis of 2008, and another  peak is found between 2005 and 2006. We could not find any literature reasoning about the peak between 2005 and 2006, but it would be interesting to analyse this further. The symmetric difference returned by the   \texttt{DD-SYM} method is more noisy.

\begin{figure}[h]
\centering
    \begin{minipage}{0.5\textwidth}
      \centering
    \includegraphics[width=\textwidth]{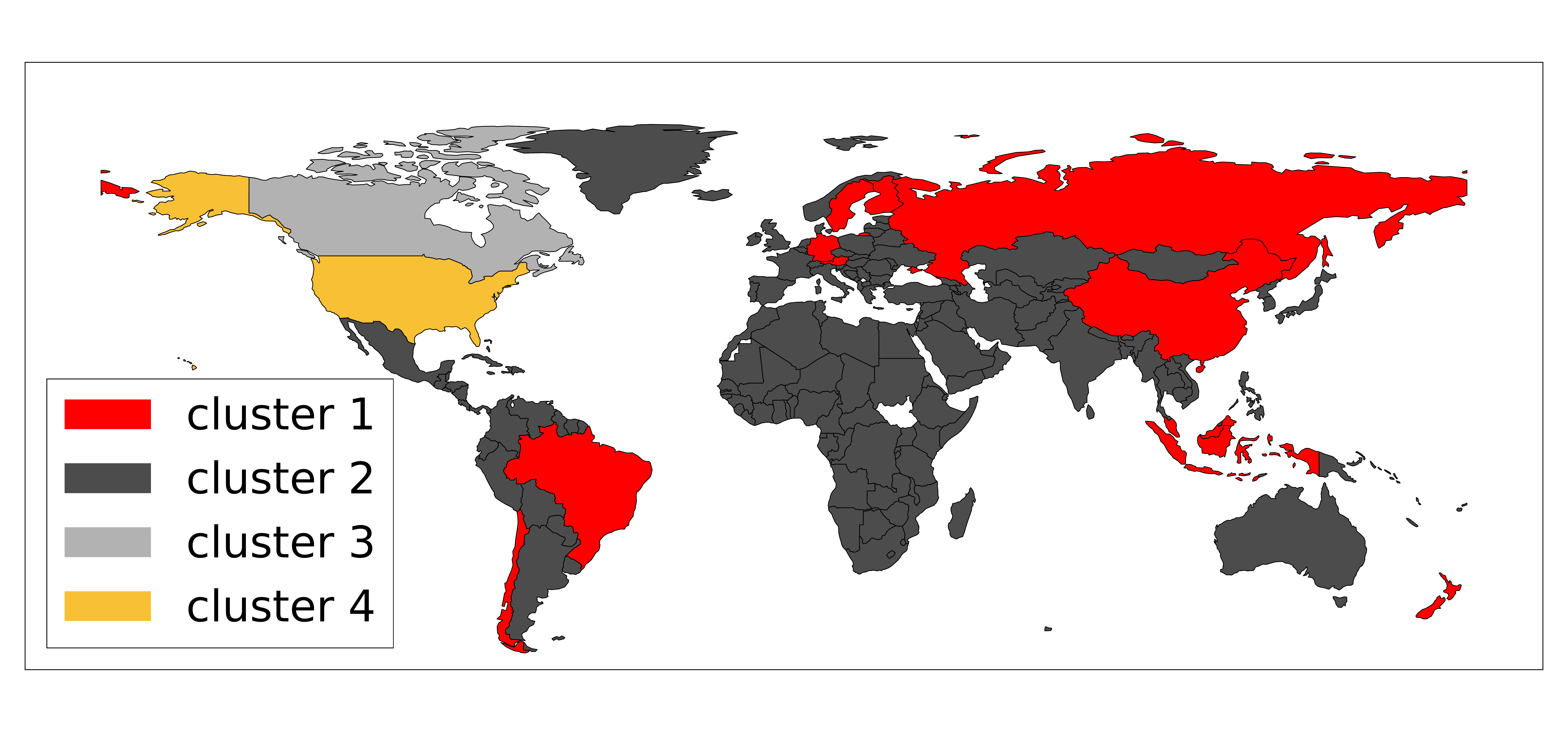}
    \vspace{-0.8cm}
    \caption*{2006}
    \end{minipage}%
    \begin{minipage}{0.5\textwidth}
      \centering  
    \includegraphics[width=\textwidth]{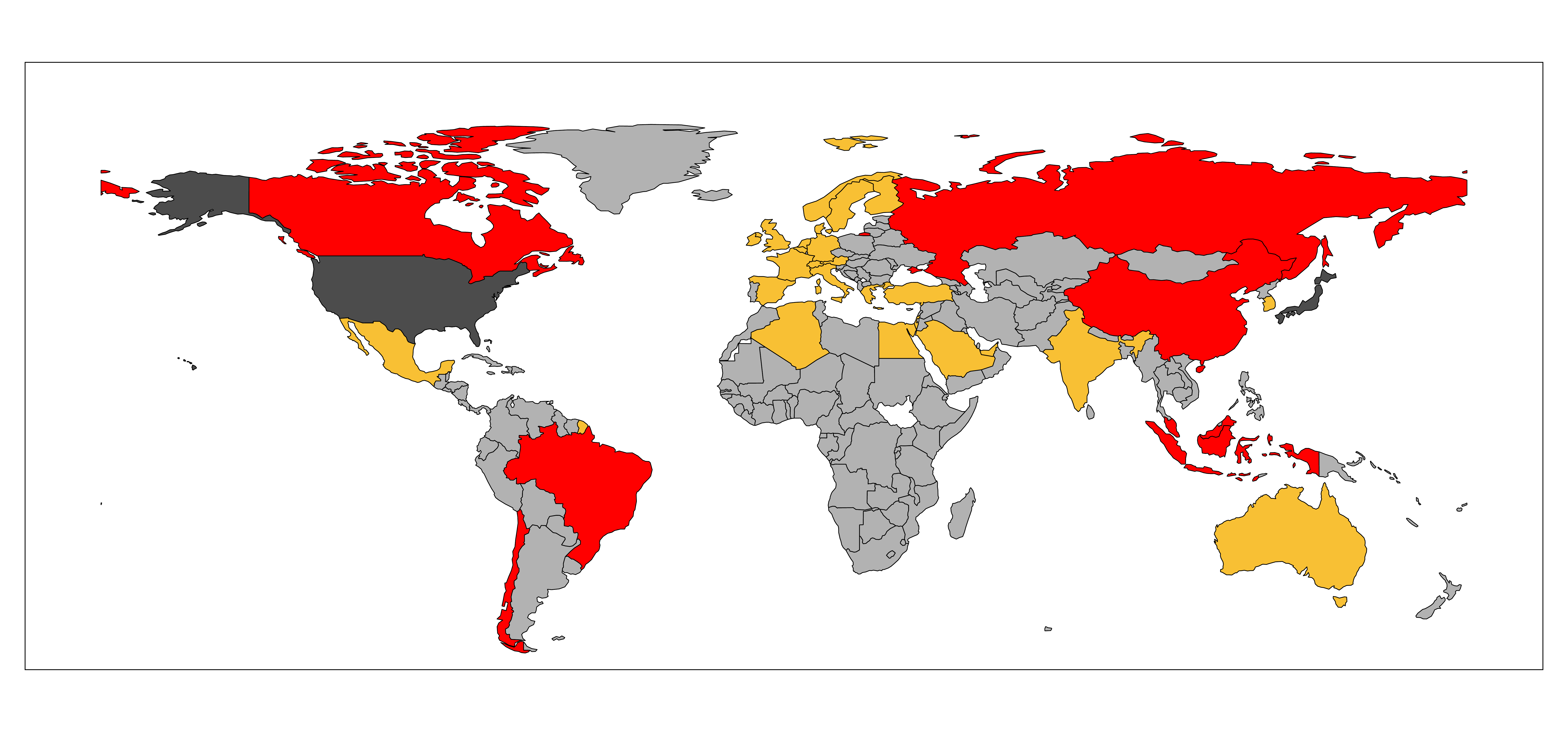}
    \vspace{-0.8cm}
    \caption*{2007}
    \end{minipage}
    \begin{minipage}{0.5\textwidth}
      \centering
    \includegraphics[width=\textwidth]{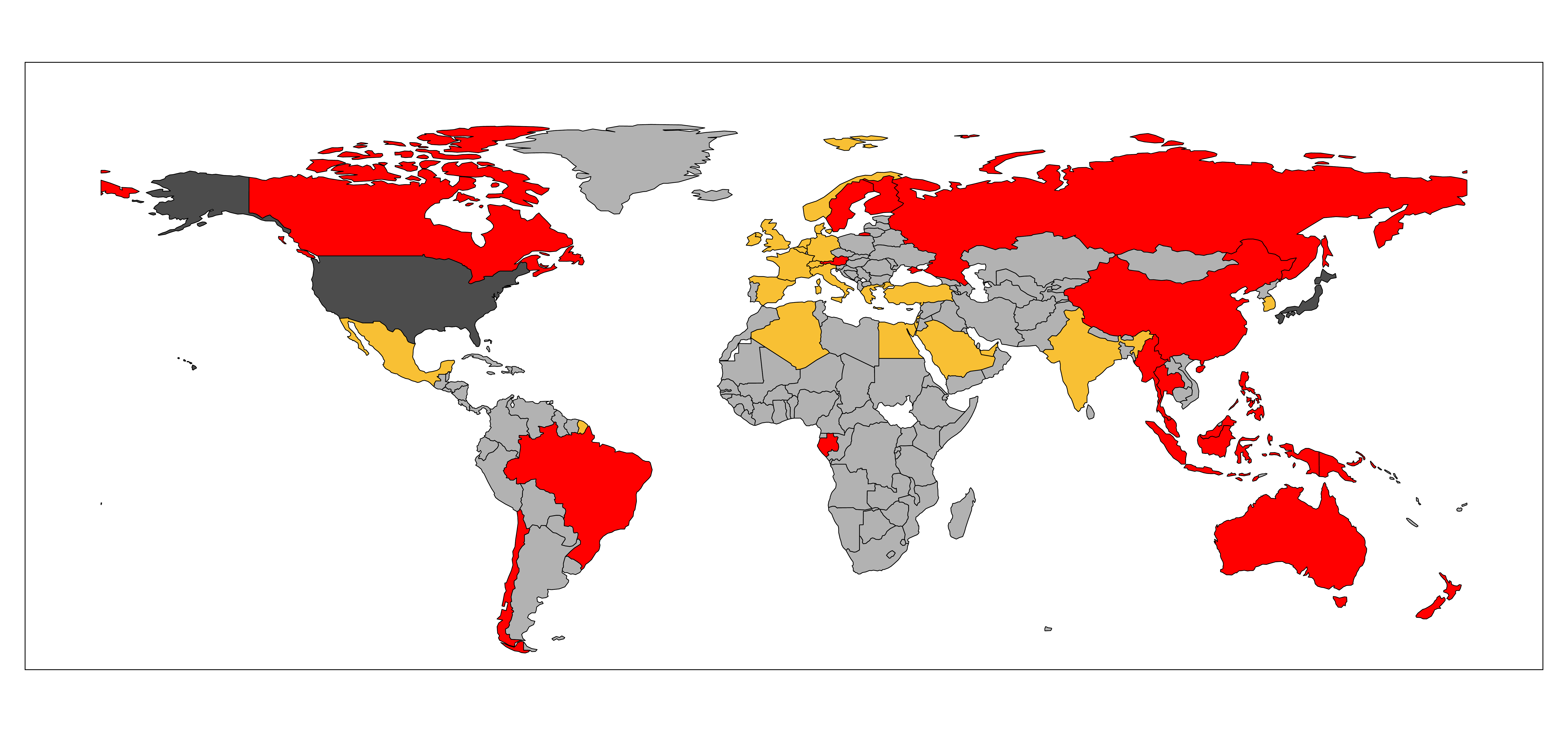}
    \vspace{-0.8cm}
    \caption*{2008}
    \end{minipage}%
    \begin{minipage}{0.5\textwidth}
      \centering
    \includegraphics[width=\textwidth]{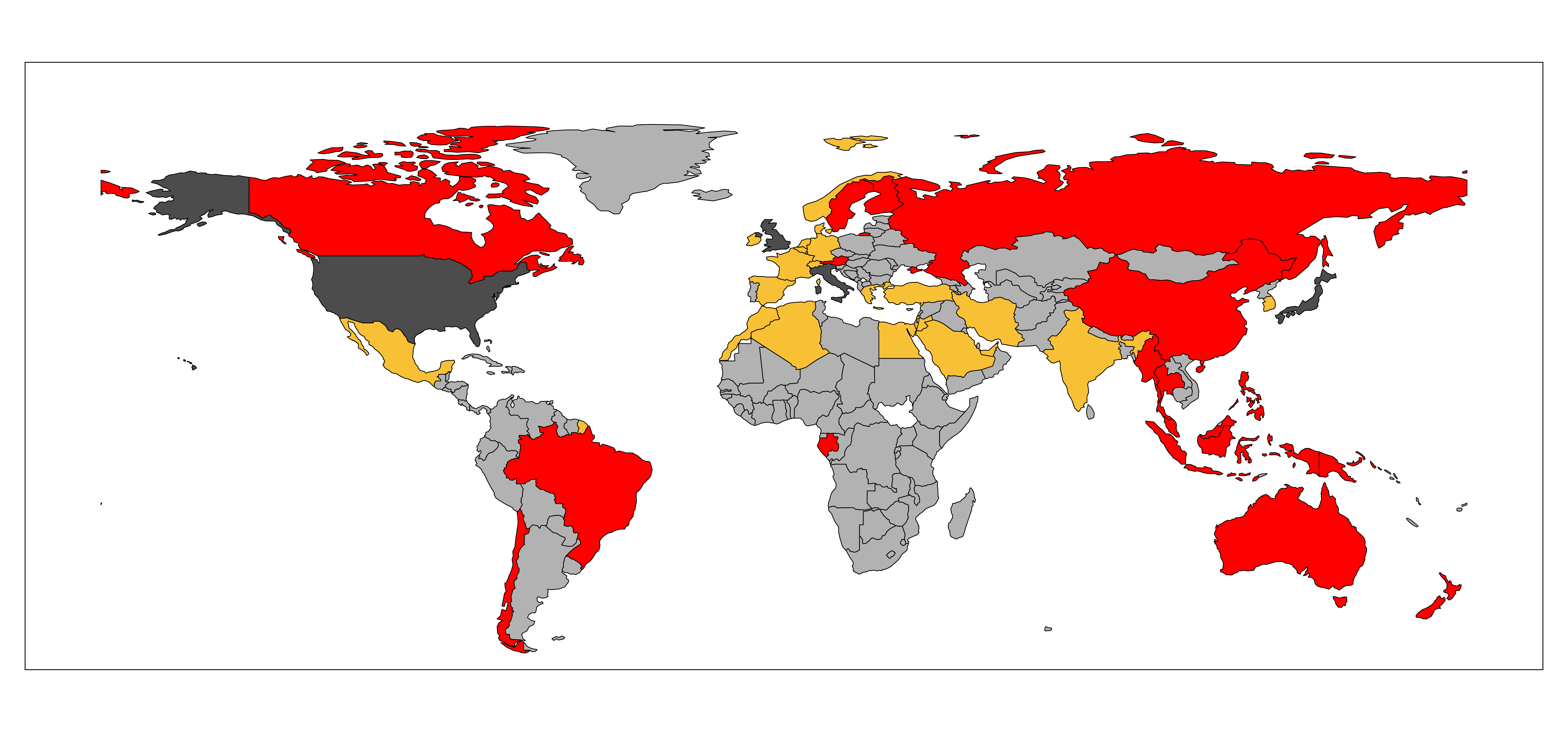}
    \vspace{-0.8cm}
    \caption*{2009}
    \end{minipage}%
\caption{Change in clustering of the IWTN over period 2006-2009 with $k=4$ using \texttt{DD-SYM} method.   Red countries form the start of the trade chain, and yellow countries the end of the trade chain. Countries   coloured white have no data.} \label{fig:wood_cluset_change_20062009_ddsym}
\end{figure}

\begin{figure}[h]
\centering
    \includegraphics[width=0.7\textwidth]{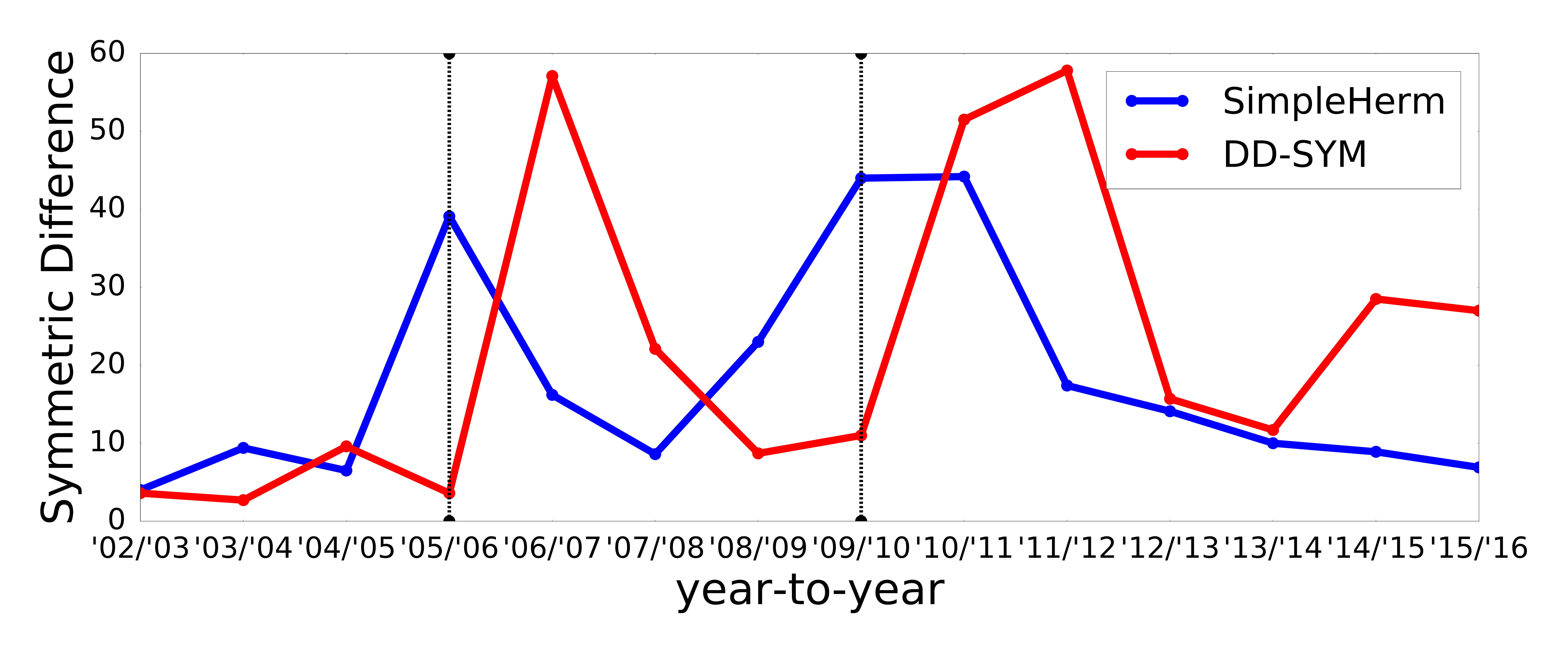}
    \caption{Comparison of the symmetric difference of the returned clusters between two consecutive years on the IWTN.}
    \label{fig:wood_symmetric_difference}
\end{figure}

\subsection{Results on Data Science for COVID-19 Dataset}\label{sec:appendix_DS4C}
The \textit{Data Science for COVID-19 Dataset} (DS4C)~\cite{DS4C} contains information about $3519$ South Korean COVID-19 cases, and we use directed edges to represent how the virus is transmitted among the individuals.  We notice that   there are only $831$ edges in the graph and there are many connected components of size $2$. To take this into account, we run our algorithm on the largest connected component of the infection graph, which consists of $67$ vertices and $66$ edges.  Applying the complex-valued Hermitian matrix and    the eigenvector associated with the smallest eigenvalue, the spectral embedding is visualised in Figure~\ref{fig:covid19_result}.


\begin{center}
\begin{figure}[h]
\centering
    \includegraphics[width=0.7\textwidth]{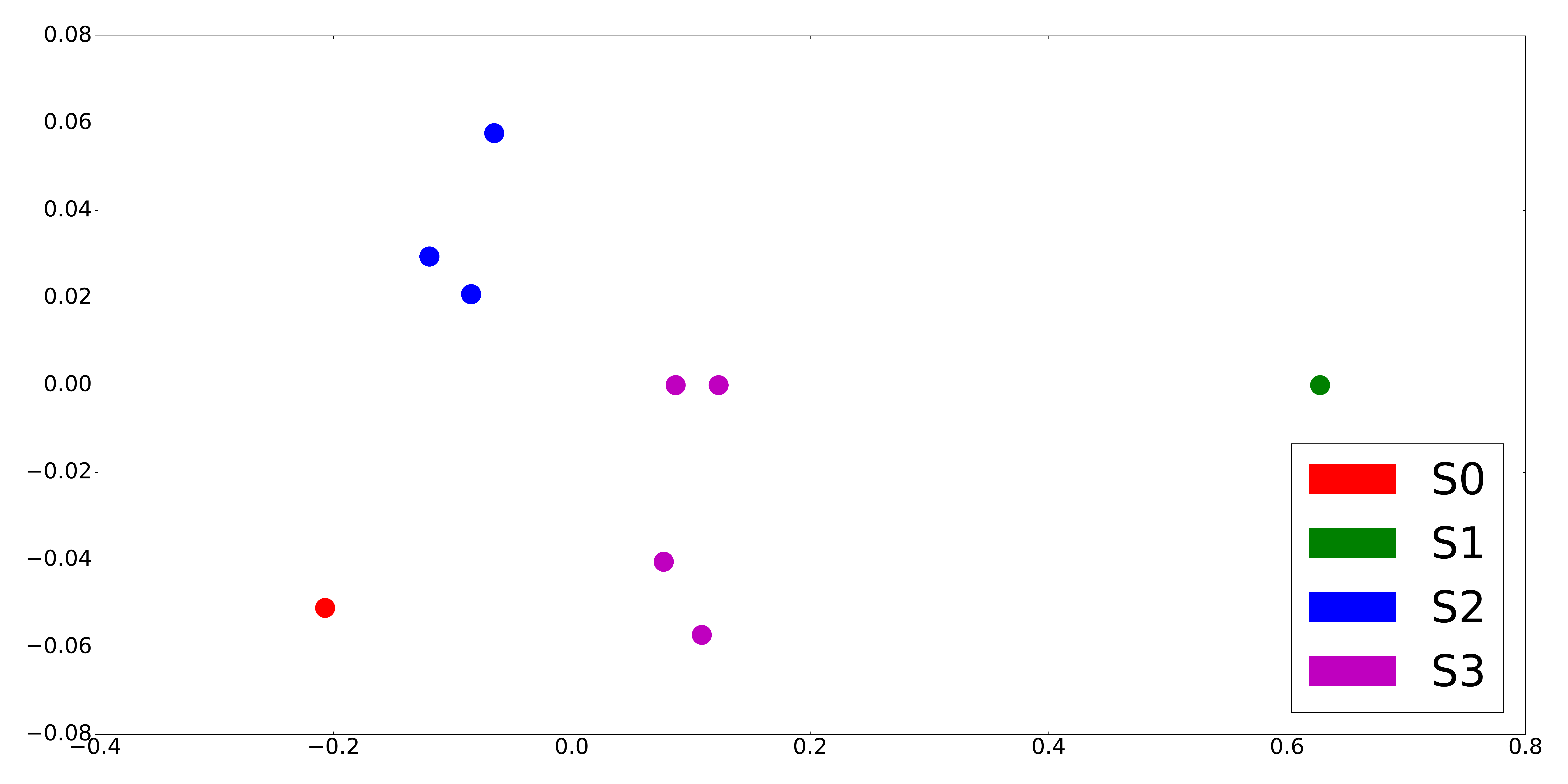}
    \caption{Clustering output on the largest connected component of the DS4C dataset, where $k=4$. Clusters are labelled according to their position in the ordering that maximises the flow ratio.}
    \label{fig:covid19_result}
\end{figure}
\end{center}

We notice several interesting facts. First of all, we do not see all the individual nodes of the graph in this embedding. This is because many
embedded points  are overlapped, which happens if they have the same in and outgoing edges. Moreover, from cluster $S_0$ to $S_1$ there is $1$ edge, from $S_1$ to $S_2$ there are $51$ edges and from $S_2$ to $S_3$ there are $5$ edges. That means there are $1+51+5 = 57$ edges that lie along the path, out of $66$ edges in total. This concludes  that our algorithm has successfully clustered the vertices such that there is a large flow ratio along the clusters.  

Secondly,  due to the limited size of the dataset,  it is difficult for us to draw a more  significant conclusion from the experiment.  However, we do notice that the   cluster $S_1$ actually consists of one individual: a super spreader. This individual infected $51$ people in cluster $S_2$. We believe that, with the development of many tracing Apps across the world and more data available in the near future, our algorithm could become a   useful tool for disease tracking and policy making.

\end{document}